\documentclass[twoside]{article}
\usepackage[accepted]{aistats2021}%,times}

\usepackage[utf8]{inputenc} % allow utf-8 input
\usepackage[T1]{fontenc}    % use 8-bit T1 fonts
\usepackage{hyperref}       % hyperlinks
\usepackage{url}            % simple URL typesetting
\usepackage{booktabs}       % professional-quality tables
\usepackage{amsfonts}       % blackboard math symbols
\usepackage{nicefrac}       % compact symbols for 1/2, etc.
\usepackage{microtype}      % microtypography
\usepackage{color}
\usepackage{graphicx}
\usepackage{subcaption}

\usepackage{amsmath}
\usepackage{amssymb}
\usepackage{amsthm}
\usepackage{mathtools}
\usepackage{xspace}
\usepackage{bm}
\usepackage{textcomp}
\usepackage{multirow}

\usepackage[round]{natbib}

\newtheorem{theorem}{Theorem}
\newtheorem{definition}{Definition}
\newtheorem{lemma}{Lemma}
\newtheorem{corollary}{Corollary}
\newtheorem{assumption}{Assumption}

\newif\ifcomments
\commentsfalse
\ifcomments

\usepackage[textsize=small,color=lightgray]{todonotes}
\paperwidth=\dimexpr \paperwidth + 10cm\relax
\paperheight=\dimexpr \paperheight + 5cm\relax
\oddsidemargin=\dimexpr\oddsidemargin + 5cm\relax
\evensidemargin=\dimexpr\evensidemargin + 5cm\relax
\marginparwidth=\dimexpr \marginparwidth + 5cm\relax
\else
\usepackage[disable]{todonotes}
\fi

\newcommand{\yuandong}[1]{\todo[fancyline, color=red!50]{\textbf{Yuandong}: #1}\ignorespaces}

\def\vw{\mathbf{w}}

\def\vx{\mathbf{x}}
\def\vy{\mathbf{y}}

\def\vu{\mathbf{u}}
\def\vv{\mathbf{v}}
\def\vg{\mathbf{g}}

\def\vf{\mathbf{f}}
\def\vp{\mathbf{p}}

\def\cD{\mathcal{D}}

\def\cO{\mathcal{O}}

\def\vzero{\boldsymbol{0}}

\def\t{\intercal}

\def\rr{\mathbb{R}}

\begin{document}

%\maketitle

\newcommand{\tsframework}{teacher-student\xspace}

\twocolumn[

\aistatstitle{Understanding Robustness in Teacher-Student Setting: A New Perspective}

\vspace{-0.3in}

\aistatsauthor{Zhuolin Yang$^*$ \And Zhaoxi Chen  \And  Tiffany (Tianhui) Cai}

\aistatsaddress{ UIUC \And Tsinghua University \And Columbia University}

\aistatsauthor{Xinyun Chen \And Bo Li \And Yuandong Tian$^*$}

\aistatsaddress{UC Berkeley \And UIUC \And  Facebook AI Research } 

\runningauthor{Zhuolin Yang$^*$, Zhaoxi Chen, Tiffany (Tianhui) Cai, Xinyun Chen, Bo Li, Yuandong Tian$^*$}
%{\small
%\setlength{\tabcolsep}{23pt}
%\begin{tabular}{lll}
%\textbf{Zhuolin Yang}$^*$ & \texttt{zhuolin5@illinois.edu} & University of Illinois at Urbana-Champaign \\ 
%\textbf{Zhaoxi Chen} & \texttt{frozen.burning@gmail.com} & Tsinghua University\\
%\textbf{Tiffany (Tianhui) Cai} & \texttt{tc3100@columbia.edu} & Columbia University \\
%\textbf{Xinyun Chen} & \texttt{xinyun.chen@berkeley.edu} & UC Berkeley \\
%\textbf{Bo Li} & \texttt{lbo@illinois.edu} & University of Illinois at Urbana-Champaign \\
%\textbf{Yuandong Tian$^*$} & \texttt{yuandong@fb.com} & Facebook AI Research
%\end{tabular}
%}

%\vspace{-0.3in}
%\aistatsauthor{}
%\aistatsaddress{}
]

\vspace{-0.1in}
\begin{abstract}
\vspace{-0.1in}
Adversarial examples have appeared as a ubiquitous property of machine learning models where bounded adversarial perturbation could mislead the models to make arbitrarily incorrect predictions.
Such examples provide a way to assess the robustness of machine learning models as well as a proxy for understanding the model training process.
There have been extensive studies trying to explain the existence of adversarial examples and provide ways to improve model robustness, e.g., adversarial training.
Different from prior works that mostly focus on models trained on datasets with predefined labels, 
we leverage the teacher-student framework and assume a teacher model, or \emph{oracle}, to provide the labels for given instances. 
In this setting, we extend~\citet{tian2019student} in the case of low-rank input data, and show that \emph{student specialization} (the trained student neuron is highly correlated with certain teacher neuron at the same layer) still happens within the input subspace, but the teacher and student nodes could \emph{differ wildly} out of the data subspace, which we conjecture leads to adversarial examples. 
Extensive experiments show that student specialization correlates strongly with model robustness in different scenarios, including students trained via standard training, adversarial training, confidence-calibrated adversarial training, and training with the robust feature dataset. Our studies could shed light on the future exploration of adversarial examples, and potential approaches to enhance model robustness via principled data augmentation.

\end{abstract}
\iffalse
 and give theoretical upper bound
 
in low-dimensional space
 
and our analysis is consistent with existing observation that the off-manifold directions in low-dimensional input are spaces where adversarial examples can be constructed.
\fi
%\vspace{-0.3in}
\section{Introduction}
%\vspace{-0.1in}
The existence of adversarial examples is an intriguing and important phenomenon in deep learning. Understanding why such examples exist can lead to (1) more robust architectures and training algorithms usable in the real world, and (2) better understanding of network training and learned representations. 

Many previous works on adversarial examples~\citep{goodfellow2014explaining,szegedy2013intriguing} focus on the standard setting of supervised classification learning in which a network is trained on a fixed dataset $\cD = \{(\vx_i, \vy_i)\}$, where $\vx_i$ is a high-dimensional input feature and $\vy_i$ is its label (continuous or discrete). 
While general, the worst-case scenario (i.e., random label $\vy_i$) may lead to exponentially many adversarial examples since every corner of the input space needs to be covered, which might never happen in practice. 

In this paper, we take a novel perspective to study adversarial examples with the \emph{teacher-student} formulation. In this setting, we have a teacher network $f^*$ as an \emph{oracle} network to provide the true label $\vy_i$ given the input $\vx_i$, i.e. $\vy_i = f^*(\vx_i)$. By definition, there is \emph{no} adversarial examples for the teacher. For a student network $f$, while $\|\vf(\vx) - \vf^*(\vx)\|$ remains small when $\vx \in \cD$, the \emph{adversarial samples} $\vx'$ for $f$ have large $\|\vf(\vx') - \vf^*(\vx')\|$ while in the local neighborhood of $\vx$.

This teacher-student assumption imposes implicit realizable constraints for $(\vx_i, \vy_i)$ pairs. Using the teacher as the \emph{reference network}, we open the black-box mapping $\vx \mapsto \vy$, and more in-depth analysis can be performed. Moreover, such a setting has interesting properties~\citep{tian2019student}: with full-rank and sufficient input, student nodes in multi-layer ReLU networks are \emph{specialized} to teacher nodes at the same layer after training (both networks have the same depth). Also, there exist unspecialized student nodes in the final trained student model. We hypothesize that the existence of such nodes is the source of the non-robustness of a trained model, which opens a new way to study robustness and adversarial samples.

In this work, we extend~\citet{tian2019student} to handle the low-rank dataset and use \textbf{N}ormalized \textbf{C}orrelation (NC) between teacher and student nodes as an additional signal to study adversarial robustness of the student network. We analyze the cause of adversarial samples and show positive correlations between NC and robustness: 
\textbf{(1)} Theoretically, we show that student specialization happens in the low-dimensional input, and specify the conditions for unspecialized nodes.
\textbf{(2)} Empirically, we show that high NC is correlated to strong adversarial robustness, verified under different scenarios such as the comparison of student network with standard training, adversarial training, adversarial training with CCAT strategy~\citep{stutz2019confidence}, and model trained with robust feature dataset~\citep{adv_feature}. 

The teacher-student framework provides a quantitative way to understand the existence of adversarial examples in low-dimensional subspace, and a quantitative measurement (i.e., node specialization) that indicates model robustness. Our analysis also confirms several existing observations about adversarial examples~\citep{adv_feature,stutz2019confidence,khoury2018geometry} from the teacher-student framework perspective.

\section{Related Works}
% \vspace{-0.5em}
\textbf{Adversarial examples.} 
Recent studies have shown that deep neural networks are vulnerable to adversarial examples, which are carefully crafted inputs aiming to mislead well-trained ML models~\citep{goodfellow2014explaining,szegedy2013intriguing}.
Since adversarial examples have raised many security concerns for ML models, different studies have been conducted to analyze its properties, such as the reasons for their existence~\citep{shamir2019simple,shi2019understanding,ilyas2019adversarial,gu2019saliency,tsipras2018robustness,kotyan2019representation}, adversarial transferability~\citep{tramer2017space,papernot2016transferability,bhagoji2017exploring}, and compactness of adversarial regions~\citep{singh2018understanding,chen2020explore,tabacof2016exploring}.
Approaches to generate such adversarial examples have also been proposed using different perturbation measurement metrics and generative models, including both $\mathcal{L}_p$ bounded and unrestricted attacks~\citep{wong2019wasserstein,bhattad2019unrestricted,xiao2018generating,xiao2018spatially, athalye2018obfuscated, vargas2019understanding}.
However, given these rich studies on adversarial examples, it remains an open question on why a small magnitude of perturbation is enough to fool a DNN model effectively and what roles the model architecture and intermediate representation play in these attacks given the complexity of a human-labeled ``natural'' dataset. We make the first attempt to investigate such questions from a different perspective, using the teacher-student framework to provide controllable constraints for the ground-truth dataset labels.

% adversarial training.
Several defense approaches have been proposed against adversarial attacks, and one of the most effective methods is \textit{adversarial training}~\citep{madry2017towards}. Different variations for adversarial training have been studied to improve its efficiency and scalability~\citep{shafahi2019adversarial,xie2020smooth}, as well as understand its limitations~\citep{zhang2019limitations,kang2019testing}.
As adversarial training has achieved promising empirical performance by improving ML robustness, we aim to leverage the \tsframework framework to provide theoretical observations on why adversarial training defends against adversarial attacks and how the intermediate representation changes after adversarial training.

\textbf{Teacher-student setting}. 
% Brief overview of use of oracles and assumed ground-truth models in statistics computer science (to make connections there)?? Premise + main results from recent student-teacher work
%in which the Theorems in~\citep{tian2019student} apply with polynomial sample complexity. 
The \tsframework setting is an old topic~\citep{engel2001statistical,saad1996dynamics,mace1998statistical,freeman1997online,gardner1989three}. Recent work has analyzed the specialization of the student nodes towards that of the teacher for 2-layer networks~\citep{goldt2019dynamics,aubin2018committee}, and 
~\citet{allen2019learning} has shown the analysis for 2 and 3 layer networks with modified SGD, batch size 1, and heavy over-parameterization. Later~\citet{tian2019student} shows that the student neuron specialization happens around SGD critical points in the lowest layer for deep ReLU networks without parametric assumption, and provides polynomial sample complexity for 2 layer ReLU networks.
% Given the popularity and the 
In this work, we use the teacher as an ``oracle'' to provide an in-depth
understanding of adversarial examples generated against the corresponding student model due to the fact that some student nodes fail to specialize fully to the teacher. 

\def\cU{\mathcal{U}}
\def\cV{\mathcal{V}}
\def\cX{\mathcal{X}}

%We will leverage the \tsframework to analyze the relationship among model robustness, neuron correlation between the teacher and student models, and the intermediate representation of student models.
%Given this line of work on analyzing the neuron specialization which is highly related with the ML vulnerabilities, 
% \vspace{-0.1in}
\begin{figure*}
    \centering
    \includegraphics[width=0.95\textwidth]{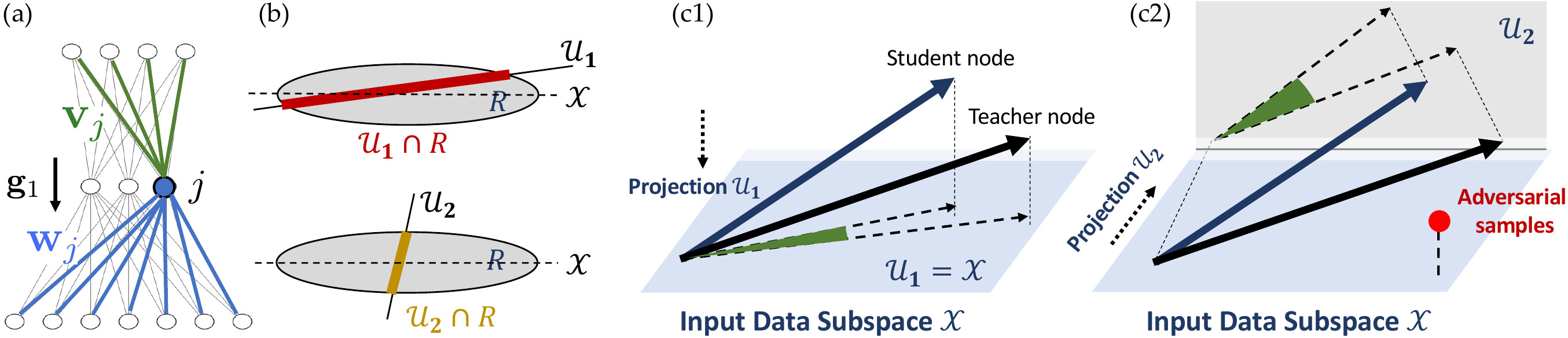}
    \vspace{-0.1in}
    \caption{\small Student Specialization in Low-rank dataset. \textbf{(a)} Setting of two-layered network (Sec.~\ref{sec:two-layer}) and notations. $\vg_1$ is the backpropagated gradient through the hidden layer. For a node/neuron $j$, its input weight is $\vw_j$ and fan-out weight $\vv_j$. \textbf{(b)} Radius of inscribed ball (1-dimension) of the intersection of a subspace $\cU$ and input data region $R$. The radius is large if $\cU$ aligns with the high-rank direction of data region. \textbf{(c1)} Within the data subspace $\cX$, the student and teacher node has small projected angle (i.e., angle between projected weights); \textbf{(c2)} If we project the weights to subspace $\cU_2 \perp \cX$, then the projected angle between student and teacher weight vectors remain large, due to limited data outside $\cX$. In this case, student response of data out of $\cX{}$ can be very different from the teacher, yielding adversarial samples.}
    \label{fig:specialization-low-rank}
    \vspace{-0.1in}
\end{figure*}
\vspace{-0.1in}

\section{Teacher-Student Setting in Low-Dimensional Input}
\vspace{-0.1in}
\subsection{Teacher network assumptions}
\vspace{-0.1in}
Let $f^*$ be the teacher and $f$ be the student. The label $\vy_i$ of each $\vx_i$ from a finite dataset $(\vx_i, \vy_i)$ is given by the teacher network $f^*$:
\begin{equation}
\setlength{\abovedisplayskip}{3pt}
\setlength{\belowdisplayskip}{3pt}
    \vy_i = f^*(\vx_i), \quad i = 1 \ldots N \label{eq:teacher-def}
\end{equation}
As an example of how teacher-student setting connects adversarial samples and robustness, in the theoretical analysis, we consider both $f^*$ and $f$ to be two-layer networks with ReLU activation and L2 loss function. 

Note that our setting is different from network distillation~\citep{hinton2015distilling}, where both teacher and student are \emph{trainable} networks given the data. In this paper, the teacher network represents \emph{an oracle} that gives the ground truth labels. Therefore, by definition, no adversarial examples exist for the teacher network.
%\tiffany{ this does not work given the definition of adversarial examples that defines it as a perturbation}

\def\proj{\mathrm{Proj}}

\subsection{Two-layer student specialization in low-rank setting}
\label{sec:two-layer}
\vspace{-0.1in}
\textbf{Notation.} For each hidden node $j$ in the student network, let $\vw_j$ be its incoming weight and $\vv_j \in \rr^C$ its fan-out weights, where $C$ is the output dimension of both teacher and student (Figure~\ref{fig:specialization-low-rank}(a)).  
Note that for $d$-dimensional input, $\vw_j := [\tilde\vw_j; b]\in\rr^{d+1}$ includes both the weight and the bias. Correspondingly, the input $\vx = [\tilde\vx; 1] \in\rr^{d+1}$, where $\tilde\vx\in \rr^d$ is the actual sample.
For teacher node, we have $\vw^*_j$ and $\vv^*_j$ respectively. Let $\vg_1$ be the backpropagated gradient at the student hidden layer and $K$ be the total number of hidden nodes (neurons) for teacher and student. 

We consider the situation where the \textbf{training has already been done}, characterized by the the condition $\|\vg_1\|_\infty < \epsilon$. Note that for mathematical convenience, the condition is \emph{stronger} than usual convergence: $\|\vg_1\|_\infty < \epsilon$ means that the gradient is small at \emph{every} data point in the data region $R$ that has \emph{infinite} samples. This ideal setting facilitates our analysis. 

One interesting phenomenon given the condition $\|\vg_1\|_\infty < \epsilon$, or in the extreme case $\vg_1 = \vzero{}$, is \emph{student specialization}~\citep{tian2019student}; that is, when the input data distribution is \emph{full-rank}, for each teacher node $j$, there exists at least one student $k$ whose weight is co-linear with the teacher: for some $\lambda > 0$, $\vw^*_j = \lambda \vw_k$. (c.f., Theorem.1 in~\citet{tian2019student}). This means that the student completely recovers the teacher's internal information upon convergence through training.

\iffalse
Furthermore, it also shows that there might exist unspecialized student node: some student node might not align with any teacher node of the same layer, but their contributions are zero. As we shall see, these unspecialized nodes play an important role to trigger adversarial examples. 

For multi-layer networks, due to over-parameterization, weights of different teacher and student nodes cannot be aligned since they might have different dimensions. However, on a common dataset, their activation patterns can be compared via normalized correlation, as done in~\citet{tian2019luck,tian2019student}.
\fi

A more interesting and realistic situation is when the input data $R$ lie in a low-dimensional space $\cX$. In this case, a perfect recovery is impossible, since there could exist multiple teachers satisfying Eqn.~\ref{eq:teacher-def}. For example, if $f^*$ is such a teacher, then for any weight $\vw^*_j$ in the lowest layer of $f^*$, there exists another teacher $f^{*'}$ with $\vw^{*'}_j = \vw^*_j + \delta\vw$, where $\delta\vw \perp \cX$, and $f^{*'}$ also satisfies Eqn.~\ref{eq:teacher-def}. Hence, we do not expect a full-specialization, but a \emph{partial} one in the input space $\cX$. Note that we use the concept of \emph{observation} between two nodes $j$ and $k$, which is a technical condition in~\citep{tian2019student}~\footnote{A node $j$ is \emph{observed} by a node $k$, if the boundary of $j$ is in the active region of $k$: $\partial E_j \cap E_k \neq \emptyset$. Here $E_j := \{\vx: \vw_j^\t\vx \ge 0\}$ is the activation region and $\partial E_j := \{\vx: \vw_j^\t\vx = 0\}$ is its boundary.}. 

\begin{theorem}[Partial Specialization for Infinite Low-Dimensional Input]
\label{thm:proj-specialization}
If the input dataset $R \subseteq \cX$, then when the gradient $\vg_1 = \vzero{}$, for each teacher node $j$ observed by any student node, there exists a student node $k$ so that $\proj_{\cX}[\vw_k] = \lambda\proj_{\cX}[\vw^*_j]$ for some $\lambda > 0$.  
\end{theorem}

See Appendix A.2 for the proof. Theorem~\ref{thm:proj-specialization} means that the weight $\vw_k$ of a specialized student node can be decomposed into two components: $\vw_k = \lambda\proj_{\cX}[\vw^*_j] + \vw^e_k$, where the first term is the useful (specialized) component of $\vw_k$. The second term $\vw^e_k$ is the component that is orthogonal to the subspace $\cX$. Note that $\vw^e_k$ is affected by initialization and $\|\vw^e_k\|$ can be arbitrarily large while not affecting the output of $f$, given its input is within $\cX$. 

For the realistic case when the gradient is small but non-zero and the input data $R$ is ``almost'' low-rank, what would happen? To characterize the low-rank structure, we consider the \emph{radius of the largest inscribed ball} in $\cU\cap R$, $r(\cU\cap R)$, with an arbitrary subspace $\cU$. If $\cU$ is aligned with the high-rank structure of $R$, then $r(\cU\cap R)$ is large, otherwise small (Figure~\ref{fig:specialization-low-rank}(b)). Here $\alpha_{jk} \coloneqq \vv^{*\t}_j\vv_k$ is the inner product between the teacher and the student fan-out weights:

\iffalse
\begin{definition}[Supporting radius of a region intersection]
For a subspace $\cU$ and a data region $R$, define $r(\cU, R)(\vx) := \max \{ r: B(\vx, r) \subseteq \cU\cap R\}$ to be the maximal radius of a ball centered at $\vx$ inscribed in $\cU \cap R$.
\end{definition}
Intuitively, $r(\cU, R)$ heavily depends on the location and orientation of $\cU$. $r$ is large if the project $\cU$ ``projected out'' low-rank dimensions, and is small otherwise. In particular, if $\cU = \cX$, then the projection completely ``projects out'' the low-rank direction and we expect good weight reconstruction in this subspace. 
\fi
\begin{theorem}[Specialization of Projected Weights in Low-Dimensional Input]
\label{thm:proj-specialization-noise}
When $\|\vg_1\|_\infty\le \epsilon$, for each teacher node $j$ observed by a student $k$, there exists a student node $k'$ so that for projected weight $\tilde\vp_{k'} := \proj_{\cU}[\tilde\vw_{k'}]$ and $\tilde\vp_j^* := \proj_{\cU}[\tilde\vw^*_j]$, their angle $\theta^\cU_{jk'} := \arccos(\tilde\vp_{k'}^\t\tilde\vp_j^*)$ satisfies $\sin(\theta^\cU_{jk'}) \le M_j(\cU)K\epsilon / \alpha_{jk}$, where $M_j(\cU) := \cO(r^{-1}(\cU\cap R\cap \partial E_j))$. %\zhuolin{what is $\vg_1$ here and how should we sample $r^{-1}(\cU, R)$?}
\end{theorem}
Please check Appendix~\ref{sec-proof2} for the proof. From Theorem~\ref{thm:proj-specialization-noise}, we can see that large radius $r(\cU\cap R)$ and large $\|\vv^*_j\|$ (and thus large $\alpha_{jk}$) yield tighter bound of specialization error. 
When the subspace $\cU$ aligns with the main direction of $R$ (or $\cX$), the inscribed radius $r$ is large, the projected angles $\theta^\cU_{jk'}$ between weight vectors are small and the alignment is good (Figure~\ref{fig:specialization-low-rank}(c1)). 
On the other hand, if $\cU \perp \cX$, the radius becomes tiny (Figure~\ref{fig:specialization-low-rank}(b)) and the projected angle has a much looser bound (Figure~\ref{fig:specialization-low-rank}(c2)). Empirically, the projected angle often remains large even after many epochs of training. 

In addition, \citet{tian2019student} pointed out that there are \emph{unspecialized} nodes, i.e., neurons that are not aligned with any teacher or student node, and their fan-out weights are zero and thus prunable. It happens in the low-dimensional input and small gradient case as well: 
\begin{corollary}[Unspecialized nodes in Low-Dimensional Input]
\label{col:proj-unspecialized}
If $\|\vg_1\|_\infty\le \epsilon$, a student node $k'$ is observed by 
%$C$
%~\bo{C? do we want to discuss it comparing with full rank data?}~\yuandong{$C$ is the output dimensions.} 
other student nodes with fan-out weights $Q = [\vv_{k_1},\vv_{k_2},\ldots,\vv_{k_C}]$, and has projected angle $\sin(\theta^\cU_{jk'}) \ge c_0$ with other teacher/student node $j$, then its fan-out weight is small: $\|\vv_{k'}\|_2 \le \|Q^{-1}\|_1 M_{k'}(\cU) K\epsilon / c_0$. %\zhuolin{Same question here and I guess $\delta$ means we can always find a $\delta$?}  
\end{corollary}

\iffalse
Note that $\alpha_{jk}$ is related to the magnitude of the fan-out weights. According to the theorem, a student node $k$ can still have large in-plane error if its $\alpha_{jk}$ becomes small. Since $\vv^*_j$ is a fixed vector, this means that the magnitude of $\vv_k$ is small.   
\fi

%\tiffany{Somewhere, we need to note bias (in addition to weight)}

\yuandong{Need more pass here.}

%\tiffany{it seems like it's not just adversarial transferability but also which teacher students converge to in general. maybe there are more connections to adversarial examples that I am missing?}

\def\tin{\mathrm{in}}
\def\tout{\mathrm{out}}

\section{Adversarial Training in the Teacher-Student Setting}
\vspace{-0.5em}
As the main contribution, we now use our teacher-student framework to analyze various adversarial phenomena. To see why adversarial training is related to the teacher-student setting, one example is the experiments in~\citet{adv_feature} that show an intriguing property of adversarial examples: using the adversarial examples $\vx'$ and their ``wrong'' labels $f^*(\vx')$ (i.e., non-robust dataset in their Sec 3.2),  we can train a student model $f$ that performs well in the original test set.  

While this sounds like ``garbage-in signal-out'', our teacher-student setting explains it naturally. The label $f^*(\vx')$ is from the output of the teacher $f^*$ on an adversarial sample $\vx'$. While this label is regarded as ``wrong’' from the dataset point of view (since $\|f^*(\vx') - f^*(\vx)\|$ is large, where $\vx$ is the data point before adversarial perturbation), from our teacher-student perspective, the input-output pair $(\vx', f^*(\vx'))$ preserves the correct mapping of the teacher, \emph{regardless} of the nature of the input data. No wonder the trained student does well on the original test set, if the teacher does well.

\iffalse
Therefore, we re-interpret the existing observations from a different perspective.

encode the correct mapping and are valid input-output from the teacher, even if $\vx'$ is adversarial and $f^*(\vx')$ is its induced wrong label, the student can still learn from it and performs well in the test set. 
\fi

With the teacher-student framework, we revise the concept of adversarial examples and analyze its properties. 

%(``wrong'' based on the dataset, not from the teacher), 

\subsection{An empirical model for learned students}
\vspace{-0.1in}
\label{sec:empirical-model}
Theorems in Sec.~\ref{sec:two-layer} tell that a learned student model on low-rank data has two properties:

\textbf{(1)} The student weight $\vw_k$ has large discrepancy from teacher weights along directions $d\perp \cX$ (Theorem~\ref{thm:proj-specialization-noise});

\textbf{(2)} If the student weight $\vw_k$ deviates from all teachers and student nodes within the data region $R$, then the magnitude of its fan-out weight is small (Corollary~\ref{col:proj-unspecialized}).\footnote{\small We leave one case for future work: two student nodes are both away from all other teacher/student nodes, and they both have strong fan-out weights.}

Note that for convenience, we omit technical conditions (e.g., the boundary needs to be observed). In the over-realization scenario, we assume that any boundary is always observed by many student nodes. 

Based on these two properties, we could come up with an \emph{empirical} model to relate a \emph{learned} student network with the teacher (here $\vw_k$ and $\vw^*_j$ are normalized):
\begin{equation}
\setlength{\abovedisplayskip}{3pt}
\setlength{\belowdisplayskip}{3pt}
  \vw_k = \vw^*_j + \epsilon_\tin \vu^\tin_k + \epsilon_{\tout} \vu^{\tout}_k \label{eq:empirical-model}
\end{equation}
where $\vu^{\tin}_k \in \cX$ and $\vu^{\tout}_k \perp \cX$ are unit vectors. $\epsilon_{\tin} = \epsilon_{\tout} = 0$ means perfect student specialization. 

Here the magnitudes of $\epsilon_{\tin}$ and $\epsilon_{\tout}$ are related to different factors. $\epsilon_{\tout}$ is related to the degree of low-rankness of the data. The more the data are rank-deficient, the smaller the supporting radius $r(\cU, R)$ for out-of-plane subspace $\cU$, and the bound becomes looser according to Theorem~\ref{thm:proj-specialization-noise}. This leads to larger $\epsilon_{\tout}$ that perturbs student node away from the teacher along the direction of out-of-distribution. 

\iffalse
Based on the previous analysis, For a student node $k$ and its aligning teacher node $j$, we have:

According to the theorems, the magnitudes of $\epsilon_{\tin}$ and $\epsilon_{\tout}$ are related to different factors. $\epsilon_{\tout}$ is related to the degree of low-rankness of the data. The more rank-deficient the data are, the smaller the supporting radius $r(\cU, R)$ for out-of-plane subspace $\cU$ and according to Theorem~\ref{thm:proj-specialization-noise}, the larger $\epsilon_{\tout}$ could be, since the bound becomes looser. 
This way, it provides higher chance (more directions off the data manifold) to construct adversarial examples, which is aligned with existing hypothesis and understanding~\citep{khoury2018geometry,ma2018characterizing}. 
\fi

On the other hand, $\epsilon_\tin$ depends on the magnitude of the fan-out weights. When the student node $k$ is \emph{unspecialized}, i.e., it strays away from teacher and other students' directions (large $\epsilon_\tin$), Corollary~\ref{col:proj-unspecialized} tells that its fan-out weight is small and therefore its influence to the output of the network is limited and/or negligible. 

\iffalse
If the node is sufficiently different from any other nodes, its effect on the output should be negligible. 
\fi
%\bo{should we discuss $\epsilon_in$ from the low rank perspective? From the fan-out weight effect perspective, it seems larger $\epsilon_in$ indicate more robust model since when it's in-plane if it's different with the teacher its effect on the network output would be 0. should we add this?}

The two unit-vectors $\vu^{\tin}_k$ and $\vu^{\tout}_k$ could be dependent on the network initialization and the training process. 

\textbf{Checking specialization of nodes}. 
There are two different ways for checking student specialization. 

\emph{Weight-check}. One method is to directly check whether $\vw^*_j = \lambda \vw_k$ for some $\lambda > 0$. While straightforward, an issue is that for intermediate layers of deep models, the input dimension of a node can be different between the teacher and an over-parameterized student. 

\emph{Activation-check}. Alternatively, we could use activation $\vf_j \in\rr^N$ computed on a given dataset of size $N$, as in~\citet{tian2019student}. By checking the Normalized Correlation between $\vf^*_j$ from the teacher and $\vf_k$ from the student, we could measure the degree of specialization.     

One short-coming for activation-check is that a perfect alignment with a low-dimensional input only tells that $\proj_{\cX}[\vw^*_j] = \lambda\proj_{\cX}[\vw_k]$, which means that $\epsilon_\tin = 0$. On the other hand, to check $\epsilon_\tout$, we would need to use data that are out of the subspace of $\cX$ (e.g., adversarial samples, adding noise to the input, etc).
\vspace{-0.1in}
\subsection{Adversarial examples in the teacher-student setting}
\vspace{-0.1in}
\label{sec:vulnerability}
Eqn.~\ref{eq:empirical-model} serves as an empirical model of the possible vulnerability of a learned student model compared to its teacher, due to $\epsilon_\tin$ and $\epsilon_\tout$. 
\textbf{First}, for a sample $\vx'$ out of the plane $\cX$, a high $\epsilon_\tout{}$ leads to large activation difference between the teacher and the student. This aligns with the existing hypothesis and understanding~\citep{khoury2018geometry,ma2018characterizing} that directions off the data manifold can be used to construct adversarial examples.
\textbf{Second}, we might also have in-plane adversarial samples that attack through $\vu^\tin$. 
\iffalse
However, since the fan-out of the weight $j$ given large $\epsilon_\tin$ is small, we might expect this attack might lead to small deviations. 
\fi

We use adversarial samples as a probe to verify our empirical model and the induced vulnerability. Since we now have a teacher network that provides the ground truth label (in addition to the data label), there are two different ways to obtain an adversarial sample.

\textbf{Oracle-adversarial}. We define \emph{oracle-adversarial} examples as follows:
%\vspace{-0.05in}
%\begin{small}{
\begin{equation}
\setlength{\abovedisplayskip}{5pt}
\setlength{\belowdisplayskip}{5pt}
    \vx' = \arg\max_{\vx'\in B(\vx, \epsilon)} L[f(\vx'), f^*(\vx')], \label{eq:backprop-adv-oracle}
\end{equation}
%\end{small}
%\vspace{-0.25in}

where $L[\cdot]$ is a loss function (e.g., $L_2$, cross-entropy, etc). $\vx'$ can be obtained by back-propagating through both $f$ and $f^*$. We call $\vx'$ the \emph{oracle-adversarial example} and $f^*(\vx')$ the \emph{oracle label}.

\textbf{Data-adversarial}. The conventional formulation of (untargeted) adversarial examples is 
\begin{equation}
\setlength{\abovedisplayskip}{5pt}
\setlength{\belowdisplayskip}{5pt}
    \vx' = \arg\max_{\vx'\in B(\vx, \epsilon)} L[f(\vx'), \vy], \label{eq:backprop-adv-data}
\end{equation}

where $\vy$ is the ground truth label from the dataset, and $L$ is commonly cross-entropy for classification. 
Here, we only obtain samples against the teacher network $f^*$ for the training set $\cD$, and assume that $f^*$ is a constant function in $B(\vx_i, \epsilon)$, where $\vx_i$ is a sample in the training set. Since $f^*$ is constant in $B(\vx_i, \epsilon)$, we use the label $\vy_i = f^*(\vx_i)$ of the \emph{original} data point $\vx_i$ when optimizing Eqn.~\ref{eq:backprop-adv-data}, and only backpropagte through the student model $f$. In this paper, we call such an adversarial examples $\vx'$ \emph{data-adversarial}. 

%\tiffany{it wasn't clear that data-adversarial and oracle-adversarial refer to $x'$ (rather than also the label) so I made that clear}

%\tiffany{\textbf{"adversarial example" in other papers refers to $x'$ where the model output is \emph{different} on $x$ and $x'$. In this work it is often (but not always) instead used to refer to refer to the argmax formulations above. }}

In the presence of the teacher network, there are two ways to do adversarial training. Let $\vx'$ be the perturbed sample. For~\emph{label-target}, we simply use the label $\vy$ of the original sample $\vx$ to update: $\theta_{t+1} \leftarrow \theta_t - \alpha\nabla_\theta L[f_{\theta_t}(\vx'), \vy]$. Alternatively, we could also use teacher output $f^*(\vx')$ as the label of $\vx'$ and update: $\theta_{t+1} \leftarrow \theta_t - \alpha\nabla_\theta L[f_{\theta_t}(\vx'), f(\vx')]$. We call it \emph{teacher-target}. It incorporates the deviation of $\vx'$ from $\vx$ and thus is more accurate than label-target.

\subsection{Why adversarial training helps model robustness?} 
Given all the previous analysis, it is now clear that by adding adversarial samples during training, we implicitly augment data region $R$ along its ``weak'' directions and thus improve student specialization (Theorem~\ref{thm:proj-specialization-noise} and Corollary~\ref{col:proj-unspecialized}). Similar effects can also be achieved by data augmentation and/or adding noise. In the next section, we will verify these findings with extensive experiments. 

\vspace{-0.05in}
\section{Experiments}
\vspace{-0.05in}
In this section, we aim to verify the strong positive correlation between the student specialization and the robustness of the student model with respect to the oracle (i.e., the teacher) in various scenarios.

We control the degree of specialization by training the student model with different epochs, as well as using \textit{adversarial training} adapted for the \tsframework framework.
In addition, we conduct studies on Confidence-Calibrated Adversarial Training (CCAT)~\citep{stutz2019confidence} to further verify the relationship between neuron specialization and model robustness. We also discussed the robust feature \citep{madry2017towards} in our \tsframework setting and left the details to Appendix~\ref{sec-robustfeature}.
\iffalse
and Robust Feature~\citep{adv_feature},
with the goal of providing explanations for these observations under the \tsframework framework.
\fi
\vspace{-0.05in}
\subsection{Experimental setup}
\vspace{-0.05in}
%~\xinyun{TODO: deprecated. Need to write the experimental details.}
We use CIFAR-10~\citep{krizhevsky2009learning} as our dataset in experiments, and consider both the teacher and student model to be the $4$-layer Conv ReLU networks.
%with $4$ hidden Conv layers followed by one Fully-Connected layer. 
We train the teacher with channel size $64-64-64-64$ at first, and then reduce it to be $45-32-32-20$ by pruning the inactivated channels\footnote{We define the channel $k$ to be inactivated by considering the norm of the fan-out weights.}. For the student model, we set it to be $1.1$x scale to the pruned teacher model (i.e. channel size $50-35-35-22$). We also investigate deeper Conv network structure by adding one more Conv layer with channel size as $64$ to further solidify our conclusion. We set each Conv layers' kernel size $s = 3$ for both teacher and student models.

In our experiments, we consider two \textbf{Standard Training (ST)} strategies. %for \tsframework Training.
\textbf{Logit training}: minimize the $\ell_2$ distance between the teacher and student' output logits.
\textbf{Label training}: minimize the cross-entropy between the student's logit and the teacher's prediction. We also consider \textbf{Adversarial Training (AT)} by training the student with \textit{oracle-adversarial examples} generated with Eq.(\ref{eq:backprop-adv-oracle}), where we apply the $40$-iteration $l_\infty$ PGD attack with perturbation scale $\epsilon=10/255$ and step size $\alpha=0.01$.

% We mainly compare two different ways to train the student models:

% \textbf{standard training (ST)}: We train the student model with the raw dataset and optimize the training loss computed from either Logit or Label training 
\vspace{-0.1in}
\subsection{Evaluation metrics}
\vspace{-0.1in}
We use the \textbf{Normalized Correlation (NC)}~\citep{tian2019luck} and its variants to measure the neuron specialization of the student to the teacher. Basically, we define $\vf_i$ to be the activations of node $i$. For student's node $k$ and teacher's node $j$, $\rho_{kj}$ is defined as the cosine similarity between the normalized activations: $\rho_{kj}= \mathbf{\tilde{f}}^\top_k \mathbf{\tilde{f}}^\ast_j$, where $\tilde{\mathbf{f}}_k = (\mathbf{f}_k - \text{mean}(\mathbf{f}_k)) / \text{std}(\mathbf{f}_k)$. Then we define the variants of the NC as follows:

\textbf{Best Normalized Correlation (BNC)} $\hat{\rho}_j$: For each teacher node $j$ in layer $l$, we find the highest NC among student's $l$-th layer nodes ($l_s$):  $\hat{\rho}_j = \max_{k \in l_s}\rho_{kj}$.

\textbf{Mean of the Best Normalized Correlation (MBNC)} $\bar\rho_l$: We compute the mean of the BNC $\hat\rho$ over teacher's $l$-th layer nodes ($l_t$):  $\bar\rho_l = \text{mean}_{j \in l_t}\hat\rho_j$.

%In order to show the general alignment between the student and the teacher, 
We also show the \textbf{Sorted BNC Curve} by sorting the BNC $\hat{\rho}$ of the teacher's nodes and concatenating the adjacents. Then we can compare students' alignment to one teacher by visualizing the curves for each layer.

\subsection{Warm-up: strong correlation between $\epsilon_\mathrm{in}$, $\epsilon_\mathrm{out}$ and normalized correlation}

First, we report $\epsilon_{\text{in}}$ and $\epsilon_{\text{out}}$ in Eqn.~\ref{eq:empirical-model} 
between the student and teacher nodes in the lowest (first Conv) layer, and study its correlation with Normalized Correlation. This is to validate our empirical model (Sec.~\ref{sec:empirical-model}) and lay the foundation of our next analysis.  

With the lowest layer's kernel size $s=3$, each input with shape $(3, 32, 32)$ can be decomposed into $30 \times 30$ patches, and each patch has $3\times 3 \times 3 = 27$ dimensions. To show the inputs' low-rank property, we perform PCA\citep{pearson1901liii} on the $27$-dimensional inputs, and the fast-decaying eigenvalues (Figure~\ref{fig:pca}) show their low-rank structure. We choose the eigenvectors with $17$ largest eigenvalues to form the basis $\mathbf{U}$ of the input distribution $\mathcal{X}$, and compute $\epsilon_{\text{in}}$ and $\epsilon_{\text{out}}$ between student node $k$ and teacher node $j$ as follows. Note that here we define $\Delta \vw_{jk} := \vw_k / \|\vw_k\|_2 - \vw^\ast_j/\|\vw^\ast_j\|_2$ (Following Sec.~\ref{sec:empirical-model}, both $\vw_k$ and $\vw^\ast_{j}$ need to be normalized): 

\vspace{-0.25in}
%\begin{small}{
\begin{align*}
\centering
\epsilon_{\text{in}}[k,j] &= \|\mathbf{UU^\top}\Delta\vw_{jk}\|_2,\\  \epsilon_{\text{out}}[k,j] &= \|(\mathbf{I-UU^\top})\Delta\vw_{jk}\|_2
\end{align*}%}
%\end{small}
\vspace{-0.25in}

To show the correlation between $\epsilon_{\text{in}}$ and NC, we use standard training and plot $(\rho_{kj}, \epsilon_{\text{in}}[k,j])$ for every pair of 
%student node $k$ and teacher node $j$ of the first Conv layer 
$k$ and $j$ in Figure~\ref{fig:pca}. We show strong negative correlation trends interpreted by Pearson score: small $\epsilon_{\text{in}}[k,j]$ indicates large NC
%: $\rho_{kj}$
. We draw the $\epsilon_{\text{in}}$ and $\epsilon_{\text{out}}$ curve by sorting the $\epsilon_{\text{in}}[k,j],\epsilon_{\text{out}}$ value between every teacher node $j$ and the student node $k$ with the highest NC. $\epsilon_{\text{in}},\epsilon_{\text{out}}$ curves show how well the student is specialized to the teacher from the in/out-plane direction.

%\yuandong{We might need to talk about the sorted BNC curve and give it a formal name. It appears many times but is not introduced} \zhuolin{fixed.}

\begin{figure}
\centering

\includegraphics[width=\linewidth]{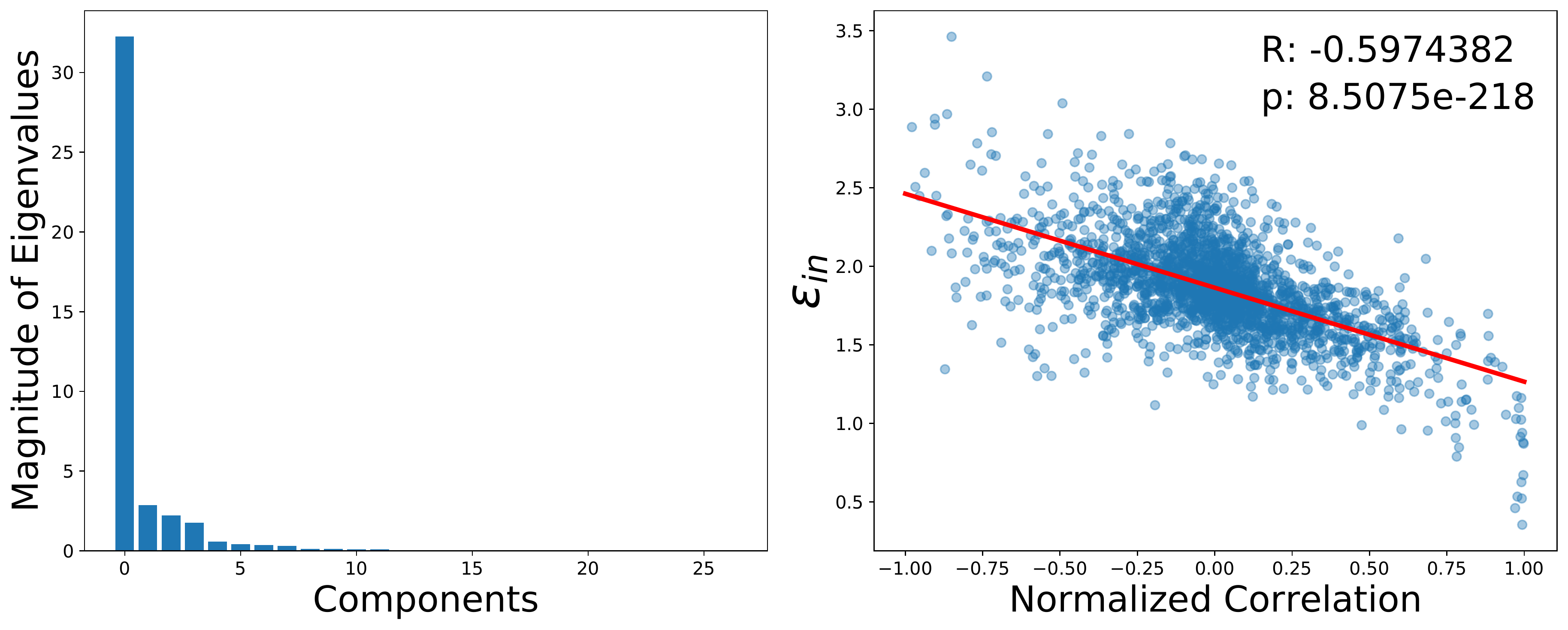}
\vspace{-2em}
\caption{\small \textbf{Left}: Magnitude of Eigenvalues on each PCA components of the input distribution $\mathcal{X}$. \textbf{Right}: Correlation between $\epsilon_{\text{in}}$ and \textbf{NC} under standard training.}
\label{fig:pca}
%\vspace{-1.5em}
\end{figure}
%\begin{figure}
%    \begin{subfigure}{.23\textwidth}
%    \includegraphics[width=\textwidth]{figs/fig-L2-final_rho-crop.pdf}
%    \end{subfigure}
%    \begin{subfigure}{.23\textwidth}
%    \includegraphics[width=\textwidth]{figs/fig-L2-r-crop.pdf}
%    \end{subfigure}
%    \begin{subfigure}{.23\textwidth}
%    \includegraphics[width=\textwidth]{figs/fig-CrossEntropy-DiscreteLabel-final_rho-crop.pdf}
%    \end{subfigure}
%    \begin{subfigure}{.23\textwidth}
%    \includegraphics[width=\textwidth]{figs/fig-CrossEntropy-DiscreteLabel-r-crop.pdf}
%    \end{subfigure}
%    \caption{\small \textbf{Left:} L2 loss and Cross Entropy loss. \textbf{Right:}. Data-adversarial and use label from unperturbed sample point.}
%    \label{fig:l2-cross-entropy}
%\end{figure}
\subsection{Adversarial training}
%We show the Adversarial Training (AT)'s effectiveness in teacher-student setting from the correlations perspective and compare it with Standard Training. Here w
In this subsection, we analyze how the Adversarial Training (AT) affects the model robustness and student specialization, measured by normalized correlation.

We run AT for different training epochs $T \in \{50, 100, 150, 200, 300\}$, and compare them to Standard Training (ST) with logit for $300$ epochs. %We show the robustness of the student models and their specialization to the teacher model. 
We check \textbf{Robust Accuracy}, defined as the ratio of successful predictions of the argmax labels of the \textit{adversarial examples}, which are generated by increasing the $\ell_2$ distance between the student and the teacher's output logits. We also show Sorted BNC and $\epsilon_{\text{in}}, \epsilon_{\text{out}}$ curves for each setting to check the node specialization.

%\begin{table}%[!htbp]
%\centering
%\vspace{-0.5em}
%\caption{\small Robust Accuracy of student models trained for different epochs (numbers in the parentheses) under Adversarial Training (AT) and Standard Training (ST).}
%\vspace{-0.5em}
%\begin{small}{
%\scalebox{0.67}{
%\begin{tabular}{c|c|c|c|c|c|c}
%\toprule
%         Model       & AT (50)               & AT (100)              & AT (150)              & AT (200)              & AT (300) & ST (300)             \\ \hline
%Robust Acc & 81.18\%      & 82.57\%       & 84.07\%       & 85.35\% & 86.81\% & 62.77\% \\ \bottomrule
%\end{tabular}}}\end{small}
%\vspace{-1.0em}
%\label{tab:attable}
%\end{table}

\begin{table*}
\centering
\small
\caption{\small Robust evaluation of student models trained for different epochs (numbers in the parentheses) under Adversarial Training (AT) and Standard Training (ST). Results are reported by the mean $\mu$ and variance $\sigma^2$ of model robust accuracy (\%) against various attacks.}
%\vspace{-0.5em}
\scalebox{0.9}{
\begin{tabular}{c|c|c|c|c|c|c|c|c|c|c|c|c}
\toprule
\multirow{2}{*}{Attacks} & \multicolumn{2}{c|}{AT (50)} & \multicolumn{2}{c|}{AT (100)} & \multicolumn{2}{c|}{AT (150)} & \multicolumn{2}{c|}{AT (200)} & \multicolumn{2}{c|}{AT (300)} & \multicolumn{2}{c}{ST (300)} \\ \cline{2-13} 
                         & $\mu$       & $\sigma^2$      & $\mu$       & $\sigma^2$       & $\mu$       & $\sigma^2$       & $\mu$       & $\sigma^2$       & $\mu$       & $\sigma^2$       & $\mu$       & $\sigma^2$       \\ \hline
$\ell_\infty$ PGD        & 78.79       & 1.1e-4        & 84.02       & 4.2e-5         & 87.57       & 2.8e-5         & 88.01       & 8.2e-6         & 88.20       & 2.9e-5         & 74.39       & 4.2e-5         \\
$\ell_2$ PGD             & 91.85       & 7.3e-6        & 95.98       & 7.0e-6         & 96.21       & 4.7e-6         & 95.95       & 4.8e-6         & 96.31       & 5.4e-6         & 94.01       & 1.0e-5         \\
$\ell_1$ PGD             & 92.30       & 9.9e-6        & 96.18       & 6.5e-6         & 96.56       & 3.1e-6         & 96.36       & 3.0e-6         & 96.59       & 3.7e-6         & 94.51       & 4.8e-6         \\
FGSM                     & 90.65       & 7.7e-6        & 95.18       & 4.7e-6         & 95.55       & 6.8e-6         & 94.87       & 5.3e-6         & 95.28       & 4.9e-6         & 91.12       & 2.1e-5         \\
CW                       & 78.89       & 1.0e-4          & 91.50       & 4.9e-5         & 91.96       & 4.3e-5         & 94.52       & 2.2e-5         & 92.63       & 1.3e-5         & 86.19       & 2.7e-5         \\
Blackbox-transfer        & 43.14       & 2.6e-5        & 45.31       & 4.7e-5         & 46.06       & 4.2e-5         & 46.68       & 3.7e-5         & 46.99       & 2.1e-5         & 43.48       & 3.2e-5         \\ \bottomrule
\end{tabular}}

\vspace{-0.5em}
\label{tab:attack-result}
\end{table*}

\begin{table}[!htbp]
\centering
%\vspace{-1em}
\small
\caption{\small Robust Accuracy (\%) of \{In-plane, Out-plane, Standard\} AT models trained for 150 epochs against \{In-plane, Out-plane, Standard\} adversarial attacks.}
%\vspace{-0.7em}
\scalebox{0.9}{
\begin{tabular}{c|c|c|c}
\toprule
               Attacks & In-plane & Out-plane & Standard  \\ \hline
AT (In-plane)  & 88.86          & 89.18           & 89.28          \\ \hline
AT (Out-plane) & 83.11          & 83.54           & 83.60          \\ \hline
AT (Standard)  & 86.87          & 87.28           & 87.18          \\ \bottomrule
\end{tabular}}
%\vspace{-1em}
\label{tab:atinoutattack}
\end{table}

\begin{figure}
\centering
\includegraphics[width=\linewidth]{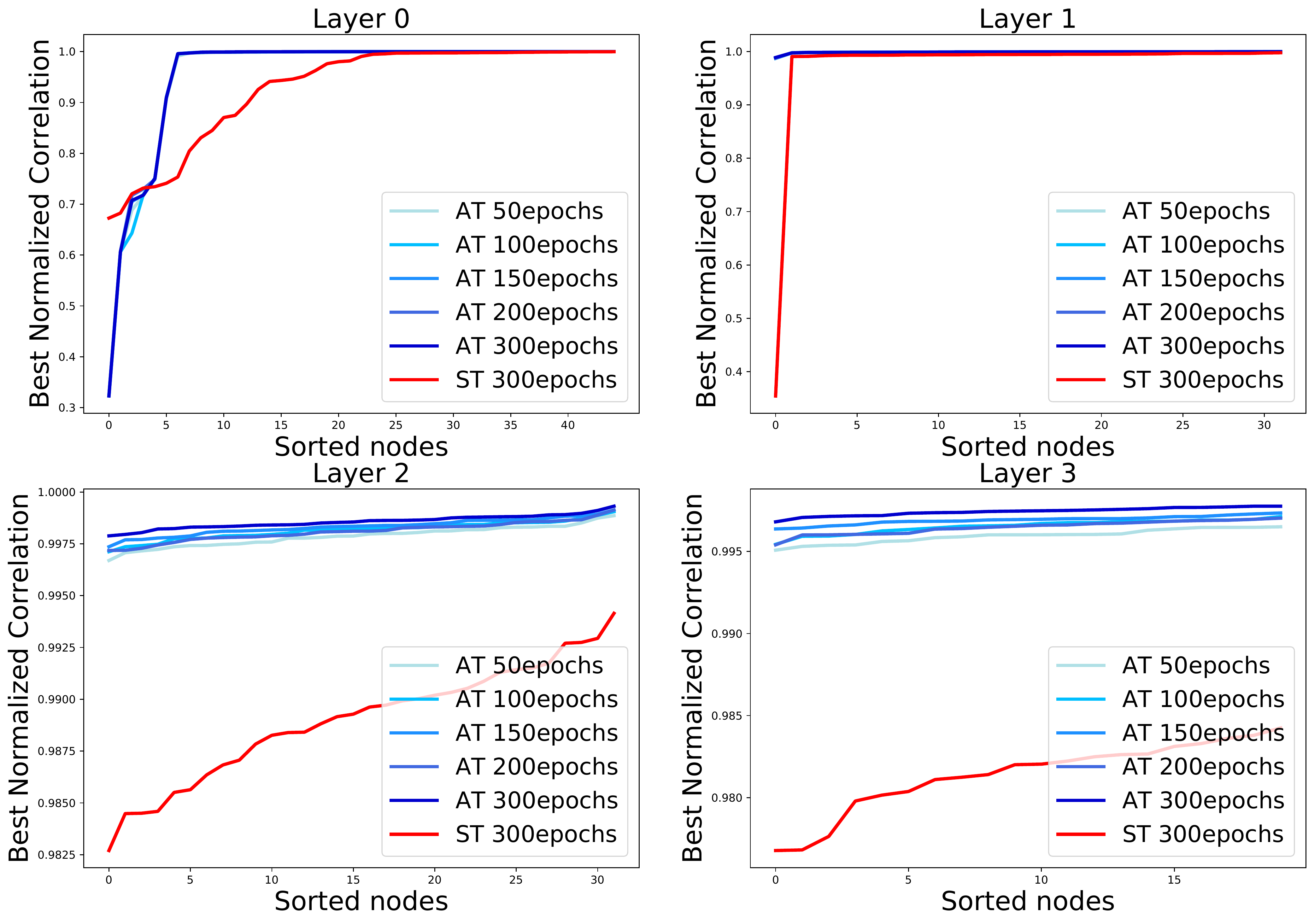}
\vspace{-1.5em}
\caption{\small Sorted BNC curve for Adversarial Training (AT) and Standard Training (ST) with logit for $300$ epochs.}
\label{fig:atvsrt}
\end{figure}

\begin{figure}
\centering
\includegraphics[width=\linewidth]{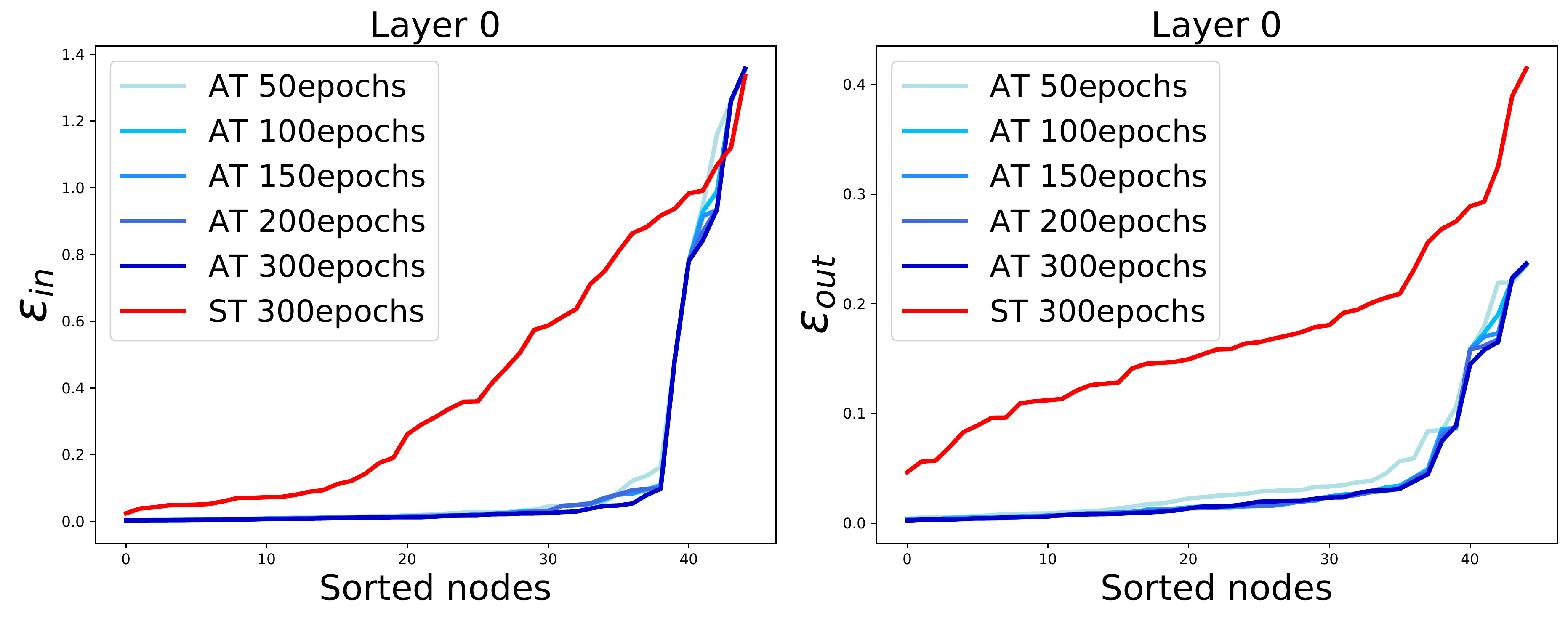}
\caption{\small ($\epsilon_{\text{in}}$, $\epsilon_{\text{out}}$) curve (lower curve means better specialization) for Adversarial Training (AT) and Standard Training (ST) with logit for $300$ epochs. AT leads to much stronger student specialization and higher robust accuracy.}
\label{fig:atvsst}
\vspace{-0.1in}
\end{figure}
%\yuandong{only show ST 150 epochs but AT at 300?}
%\zhuolin{Fixed}

\begin{figure*}
\centering
%\vspace{-0.5em}
\includegraphics[width=\linewidth]{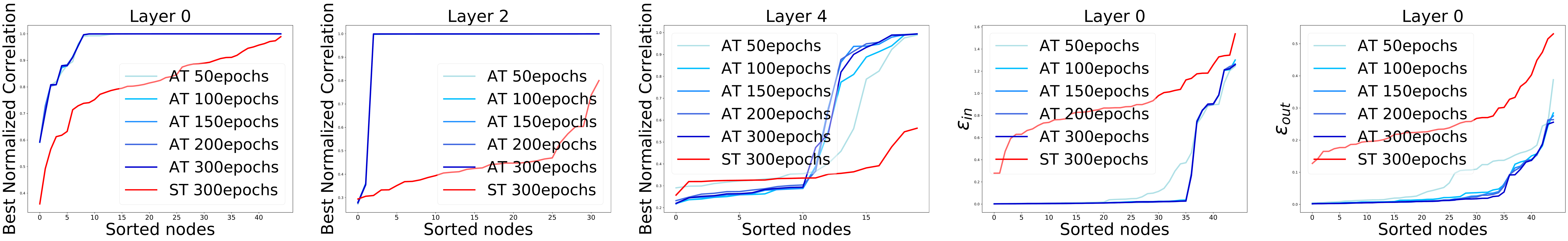}
\caption{\small \{\textbf{Left}: Sorted BNC curve on Layer 0, 2, 4; \textbf{Right}: $\epsilon_{\text{in}}$ and $\epsilon_{\text{out}}$ curve\} using a deeper Conv network architecture.}
\label{fig:deepnet}
%\vspace{-1.5em}
\end{figure*}

\begin{figure}
\centering
\vspace{-0.2in}
\includegraphics[width=\linewidth]{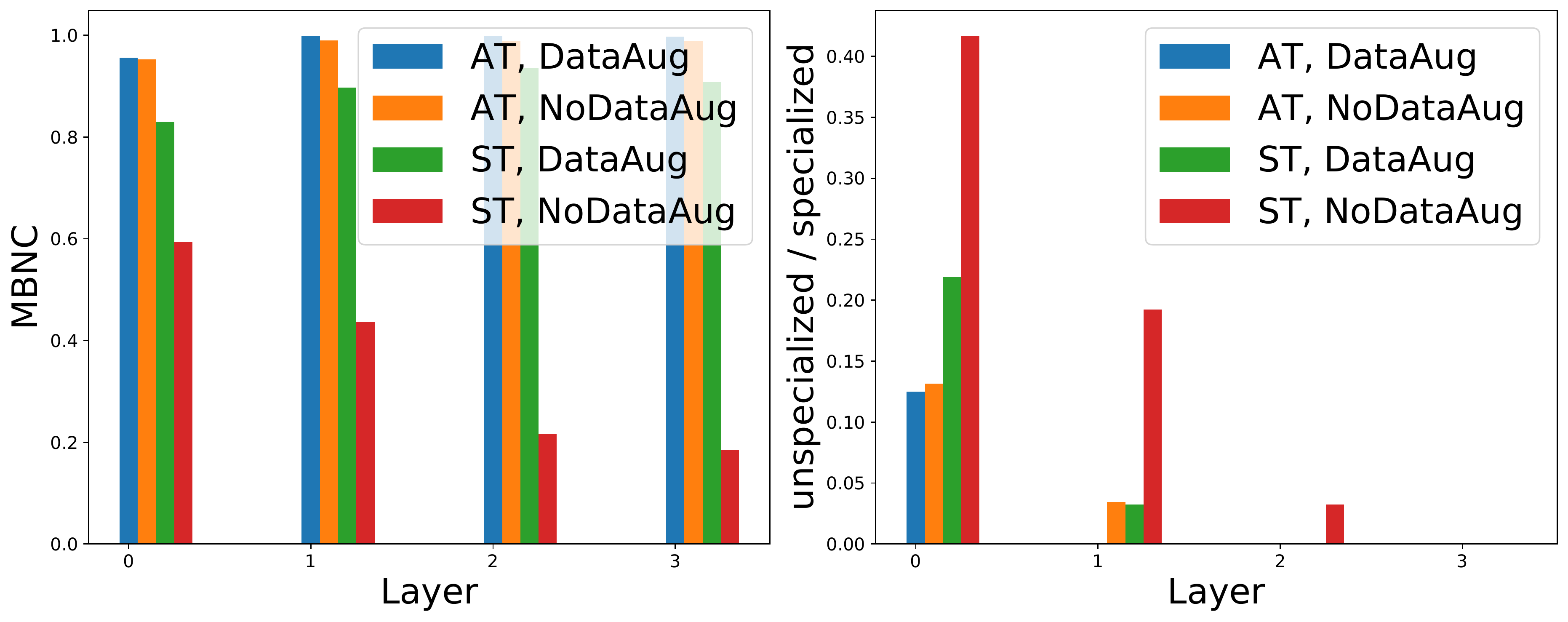}
\vspace{-0.2in}
\caption{\small Comparison between Adversarial Training (AT) and Standard Training (ST) with/without data augmentation. \textbf{Left}: The MBNC value $\hat{\rho}$ for every layer. \textbf{Right}: The ratio of the number of the unspecialized nodes divided by the number of the specialized nodes in every layer.}
\vspace{-0.1in}
\label{fig:aug}
\end{figure}

\begin{figure}
\centering
%\vspace{-0.2in}
\includegraphics[width=\linewidth]{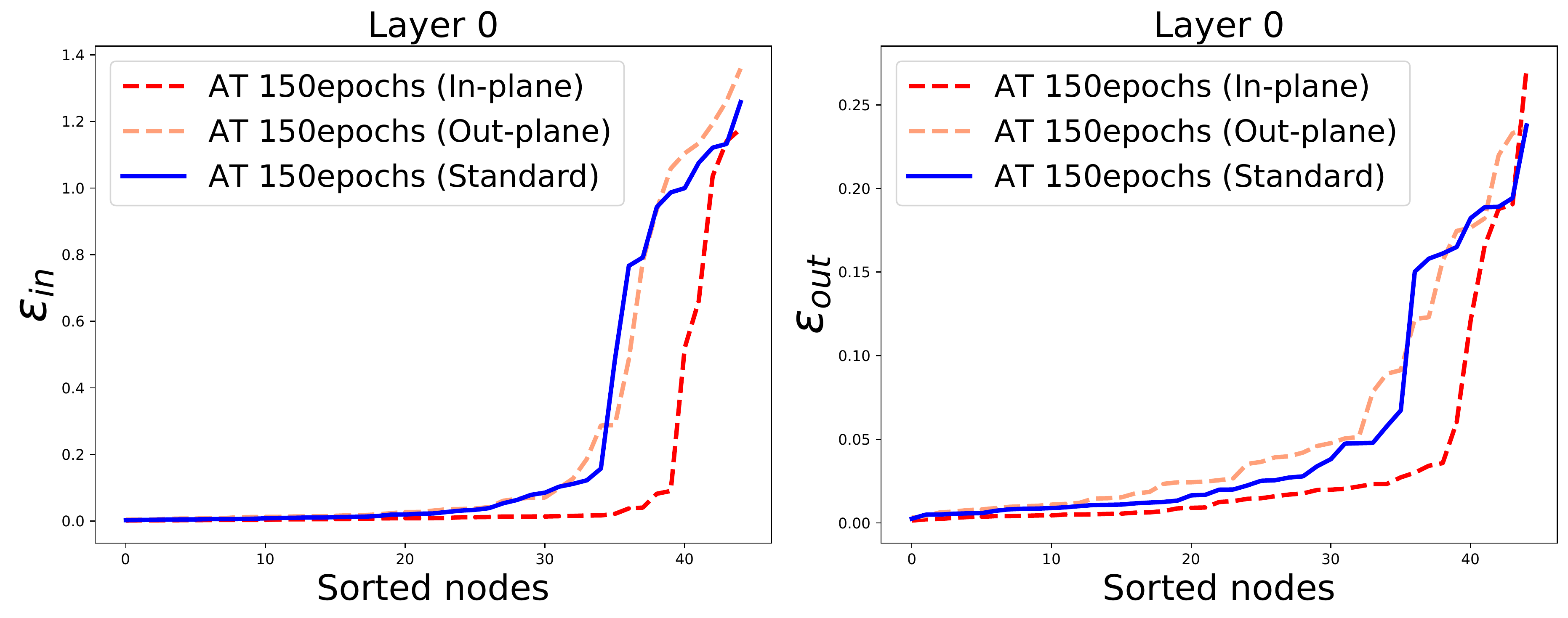}
%\vspace{-0.2in}
\caption{\small ($\epsilon_{\text{in}}$, $\epsilon_{\text{out}}$) curve for \{in-plane, out-plane, standard\} Adversarial Training (AT) with $150$ epochs.}
\label{fig:atinout}
%\vspace{-0.05in}
\end{figure}
\begin{figure}
\vspace{-0.3in}
\centering
    \includegraphics[width=0.48\textwidth]{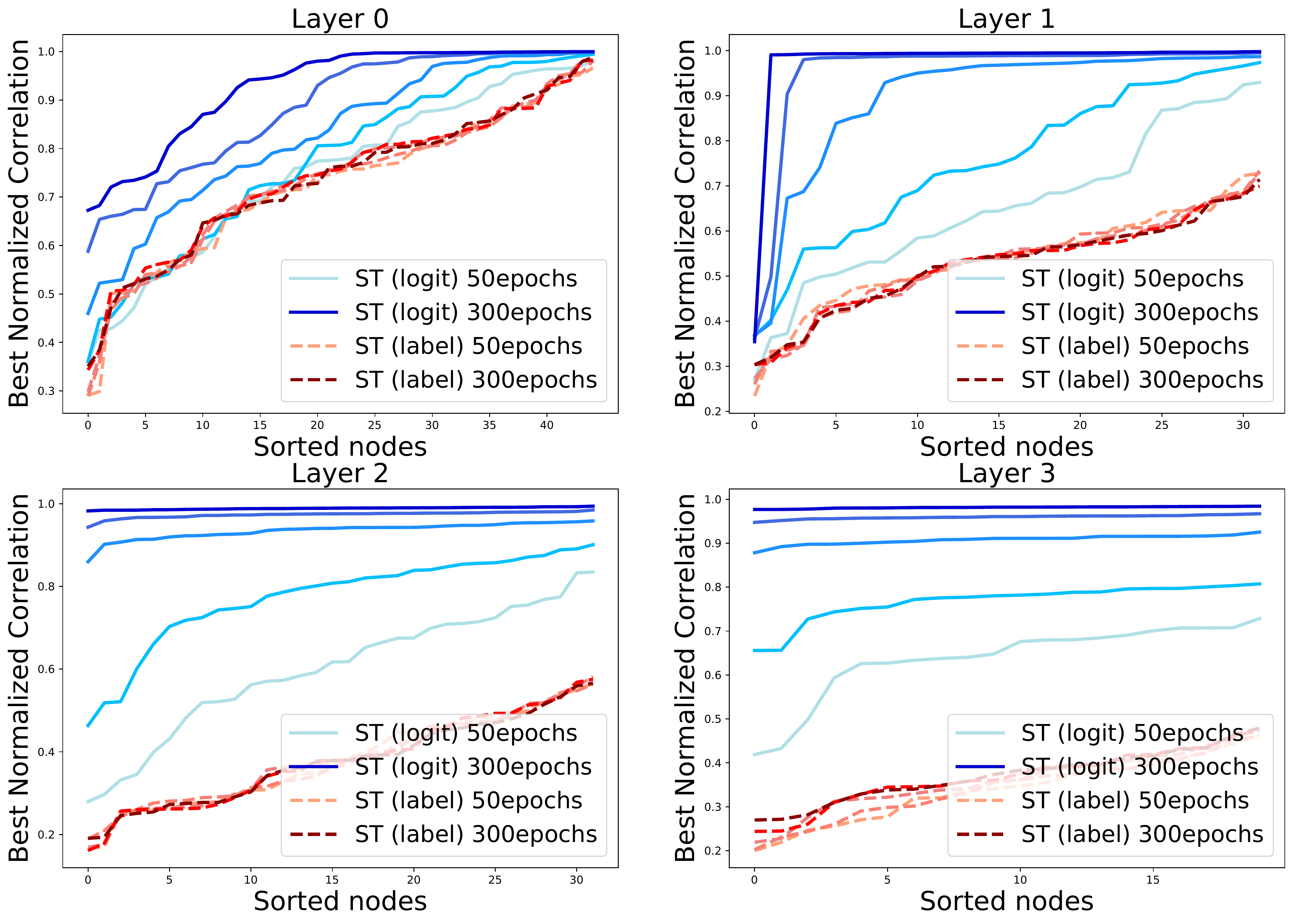}
    \vspace{-0.3in}
%\yuandongil{In Fig.~\ref{fig:logitepochs}, shall we use rainbow color to show the progress of training? current colormap can be confusing.}
    \caption{\small Sorted BNC curves (the higher means better specialization) for Standard Training (ST) with logit or label for different epochs on CIFAR-10. Logit training leads to much stronger specialization across all layers. Solid \textcolor{blue}{\textbf{Blue}} curves refer to the logit training and dashed \textcolor{red}{\textbf{Red}} curves to the label training. Color changed from light to dark with more training epochs.}
    %    \vspace{-2em}
    \label{fig:logitepochs}
%\end{figure}
\vspace{1em}
%\begin{figure}[!t]
%    \vspace{-0.15in}
\centering
    \includegraphics[width=0.48\textwidth]{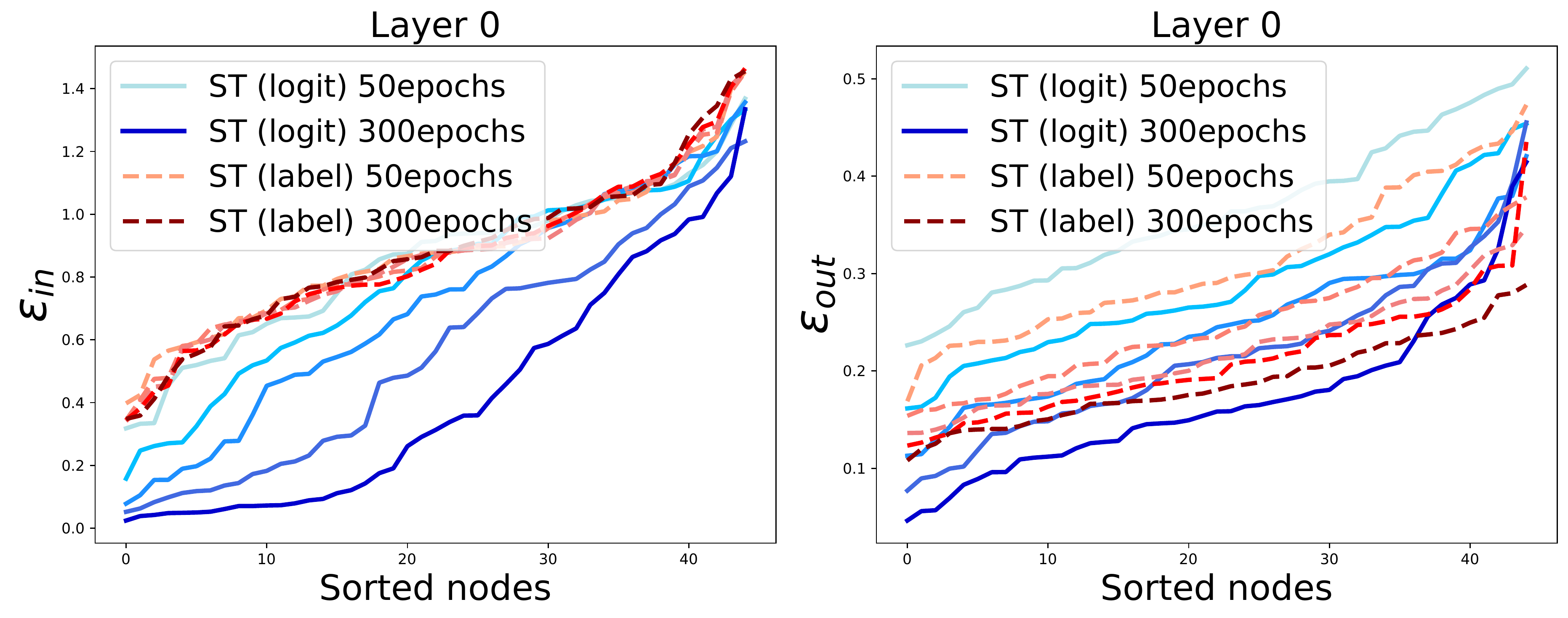}
    \vspace{-0.3in}
    \caption{\small ($\epsilon_{\text{in}}$, $\epsilon_{\text{out}}$) curves for Standard Training (ST) with logit or label for different epochs on CIFAR-10. Solid \textcolor{blue}{\textbf{Blue}} curves refer to the logit training and dashed \textcolor{red}{\textbf{Red}} curves to the label training (which reduces $\epsilon_\tout$ more). Colors are changed from light to dark with more training epochs.}
    \vspace{-0.5em}
    \label{fig:logitepochs2}
\end{figure}

\begin{figure}
\vspace{-0.2em}
\centering
    \includegraphics[width=0.48\textwidth]{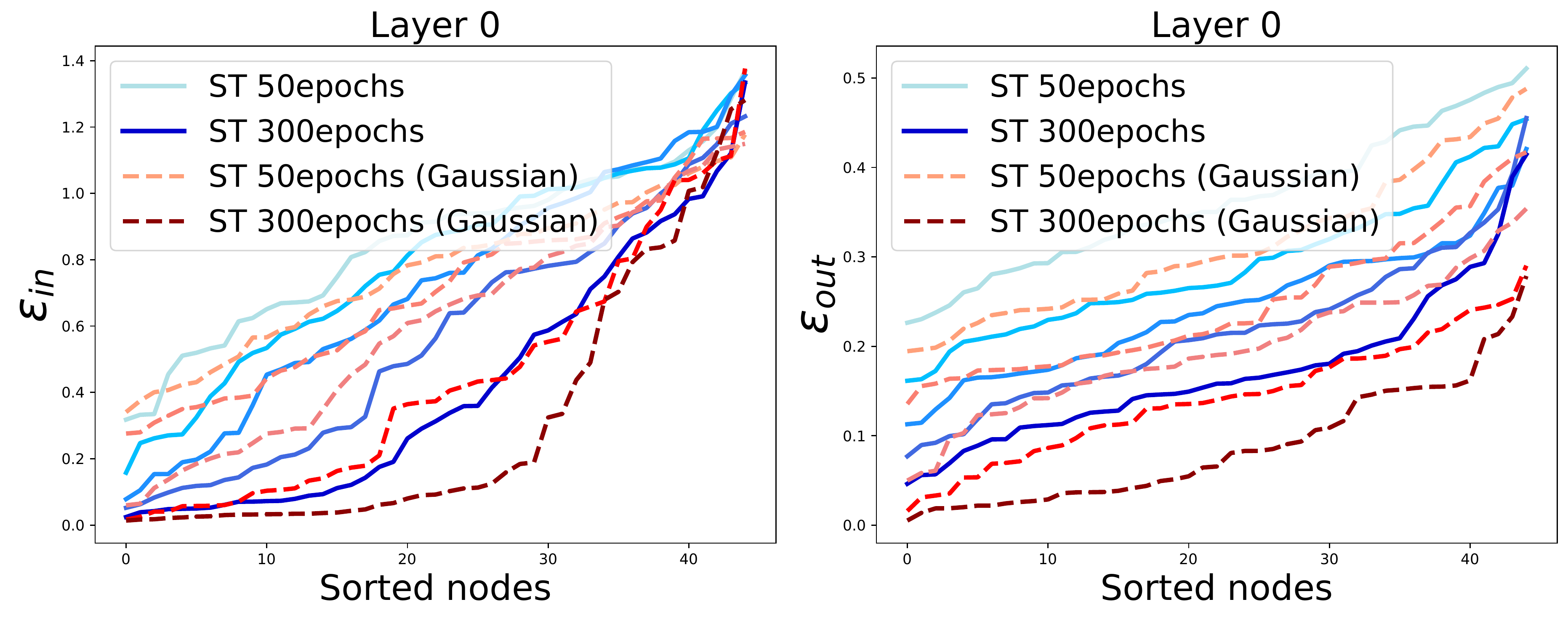}
    \vspace{-0.3in}
    \caption{\small ($\epsilon_{\text{in}}$, $\epsilon_{\text{out}}$) curves for ST (logit) with or without Gaussian augmentation for different epochs on CIFAR-10. Solid \textcolor{blue}{\textbf{Blue}} curves refer to ST (logit) without Gaussian and dashed \textcolor{red}{\textbf{Red}} curves refer to ST (logit) with Gaussian. Adding Gaussian leads to better specialization. Colors are changed from light to dark with more training epochs.}
    \vspace{-1.5em}
    \label{fig:highrank}
\end{figure}

We conduct various types of attacks to generate adversarial examples: $\{\ell_1, \ell_2, \ell_\infty\}$ optimization based PGD attack~\citep{madry2017towards}, FGSM attack~\citep{goodfellow2014explaining}, CW attack~\citep{carlini2017towards} and Blackbox-transfer attack using a surrogate model trained independently. We run robustness evaluation multiple times to compute statistical confident robust accuracy with mean $\mu$ and variance $\sigma^2$. From Table~\ref{tab:attack-result}, we can see AT model's robustness increases with epochs and surpasses the 300 epochs ST (logit) model's even at 50 epochs. Figure~\ref{fig:atvsrt} and~\ref{fig:atvsst} show the neuron specialization of the student by plotting $\epsilon_{\text{in}}, \epsilon_{\text{out}}$ and Sorted BNC  curve, where AT models achieve \emph{much better specialization} than ST models by reducing $\epsilon_{\text{in}}$ and $\epsilon_{\text{out}}$ drastically in the first few epochs. For $5$-layer deeper Conv network, we also observe similar results as shown in Figure~\ref{fig:deepnet}.%This shows robustness and neuron specialization are highly correlated and AT can be an effective way to achieve both. 

We also evaluate how in-plane AT and out-plane AT affect student’s specialization separately. To disentangle them, for each instance $\vx$, we apply the standard $\ell_\infty$ PGD attack twice with different initialization to obtain $\vx'_{\text{in}}$ and $\vx'_{\text{out}}$, while $\vx'_{\text{in}}$ has a smaller distance to the input subspace.
%and larger distance for $x'_{\text{out}}$.
We use $\vx'_{\text{in}}$ to train the in-plane AT model and $\vx'_{\text{out}}$ for out-plane AT model, and we evaluate each model's Robust Accuracy against the in-plane attack, out-plane attack, or both (i.e., the standard attack).
From Table~\ref{tab:atinoutattack},
%by including the standard AT model and adversarial attack as well. 
we find the in-plane attack can be more severe causing model's vulnerability, so in-plane AT models achieve better robustness.
%by training on the in-plane adversarial examples. 
In Figure~\ref{fig:atinout}, the plots for $\epsilon_{\text{in}}$ and $\epsilon_{\text{out}}$ indicate that the in-plane AT model leads to better specialization from both in-plane and out-plane directions. 

We evaluate the \textbf{Mean of the Best Normalized Correlation (MBNC)} and 
%the ratio of the number of \emph{unspecialized} node to \emph{specialized} node
the \emph{unspecialized}/\emph{specialized} ratio\footnote{We consider the node to be \emph{unspecialized} if NC is smaller than 0.8, and \emph{specialized} if NC is larger than 0.9.} of AT and ST models trained for 300 epochs in Figure~\ref{fig:aug}. We observe that AT model could achieve higher MBNC value by forcing more student nodes to be specialized to teacher nodes, and the traditional Data Augmentation method (RandomCrop, HorizontalFlip, Rotation) could improve neuron specialization as well.% as shown in the ablation study.

\subsection{Standard training}%\zhuolin{I commented out the "different epochs" here}% with different epochs}
%\vspace{-0.1in}
\label{sec-standard}

%\yuandong{For Fig.~\ref{fig:logitepochs}, we might use dash red line for labeled training.}
%\zhuolin{In this case, the legends couldn't have enough space or the font need to be quite small}

%\yuandong{Interesting.. ST 200 epochs has much stronger robust accuracy (53.24\%) than ST 150 epochs (35.88\%)?}
%\zhuolin{Yes. the robust accuracy increases significantly after 200 epochs}

In this subsection, we continue to study the correlation between model robustness and specialization to the teacher in Standard Training (ST) with logit or label's supervision. Our analyis is performed at different training epochs $T\in\{50, 100, 150, 200, 300\}$. 

% The attack's perturbation scale $\epsilon = 10/255$ and iterations equals to $40$ with step size $\alpha = 0.01$. 

\begin{table}[!htbp]
\centering
%\vspace{-1em}
\caption{\small Robust Accuracy (\%) of student models trained for different epochs (numbers in the parentheses) under Standard Training (ST) with logit or label supervision.}
%\vspace{-0.9em}
\begin{small}
\scalebox{0.77}{
\begin{tabular}{c|c|c|c|c|c}
\toprule
        Robust Acc        & ST (50)               & ST (100)              & ST (150)              & ST (200)              & ST (300)             \\ \hline
 Logit training & 23.12 & 30.72 & 36.72 & 48.52 & 62.77 \\ \hline
 Label training & 19.08 & 20.81 & 22.34 & 23.42 & 25.79 \\ \hline
\end{tabular}}
\end{small}
%\vspace{-1.5em}
\label{tab:epochsra}
\end{table}

From Table~\ref{tab:epochsra} and Figure~\ref{fig:logitepochs}, we can observe both robustness and specialization of ST models improved with training. However, when training with the same epochs, ST (label) model is worse than ST (logit) model from both robustness and specialization perspectives.

Moreover, in Figure~\ref{fig:logitepochs2}, we show the specialization of ST (logit) model and ST (label) model from the in-plane and out-plane directions. Interestingly, ST with label does not improve the in-plane specialization. In contrast, ST with logit leads to specialization on both in-plane and out-plane aspects.

To check the low-rank property of input distribution $\mathcal{X}$, we add the $d$-dimensional Gaussian noise $\bm{\epsilon}\sim\mathcal{N}(0,\sigma^2\bm{I}_d), \sigma=0.1$ on the input instances during the ST with logit, and we present $\epsilon_{\text{in}}, \epsilon_{\text{out}}$ curves in Figure~\ref{fig:highrank}. %difference between the settings: ST (logit) with or without Gaussian augmentation. 
From Figure~\ref{fig:highrank}, we can observe that by training with the high-rank input instances, the student can be more specialized to the teacher from both in-plane and out-plane directions. Meanwhile, the low-rank property brings the risk facing the out-plane adversarial examples.

\noindent\textit{Remarks.} 
%The strong correlation between the robustness and neuron specialization of student models 
%We claim the neuron specialization can be the strong indicator for model robustness by showing the strong correlation. 
We suggest the existence of adversarial examples is due to student's \emph{unspecialized} neurons (large $\epsilon_{\text{in}}$). During training, ST decreases student nodes' $\epsilon_{\text{in}}$, improves neuron specialization and therefore leads to better robustness.
Also, comparing to ST with label, ST with logit can leverage the additional direction information from the teacher output, and achieve better neuron specialization and robustness, which verifies our claim about the strong correlation between robustness and specialization.%again verifies that the neuron specialization could be a strong indicator for the model robustness.
%\yuandong{Shall we be more conservative?}\bo{updated}

% ~\xinyun{TODO: add the discussion of CCAT experiments.}

%\begin{figure}
%\centering

%\includegraphics[width=\linewidth]{figs/fig-lablevslogit.png}
%\vspace{-1em}

%\caption{\small Regular logit training and label training correlations comparison for different epochs on CIFAR-10 dataset. ST refers to $\textbf{standard training}$. ~\xinyun{TODO: @Zhuolin, find your results of robust accuracies related to this setting.}}
%\label{fig:logitvsdiscrete}
%\end{figure}

% We also make the comparison between the Regular Logit training and the Regular Label training. 
% From Table~\ref{tab:epochsra}, Fig~\ref{fig:logitepochs} we can see, the student model with label training has much lower correlations to teacher model than with logit training. This phenomenon emerges in robustness perspective as well which indicates that the student model who utilize the teacher model's nodes vector direction information during the training procedure can have better alignment. Discrete label information can not provide enough help for student model to learn how to specialize the nodes inside teacher model well.

\vspace{-0.1in}
\subsection{Analysis of Confidence-Calibrated Adversarial Training}
%\vspace{-0.1in}
\begin{figure}
\centering
%\vspace{-0.2in}
\includegraphics[width=\linewidth]{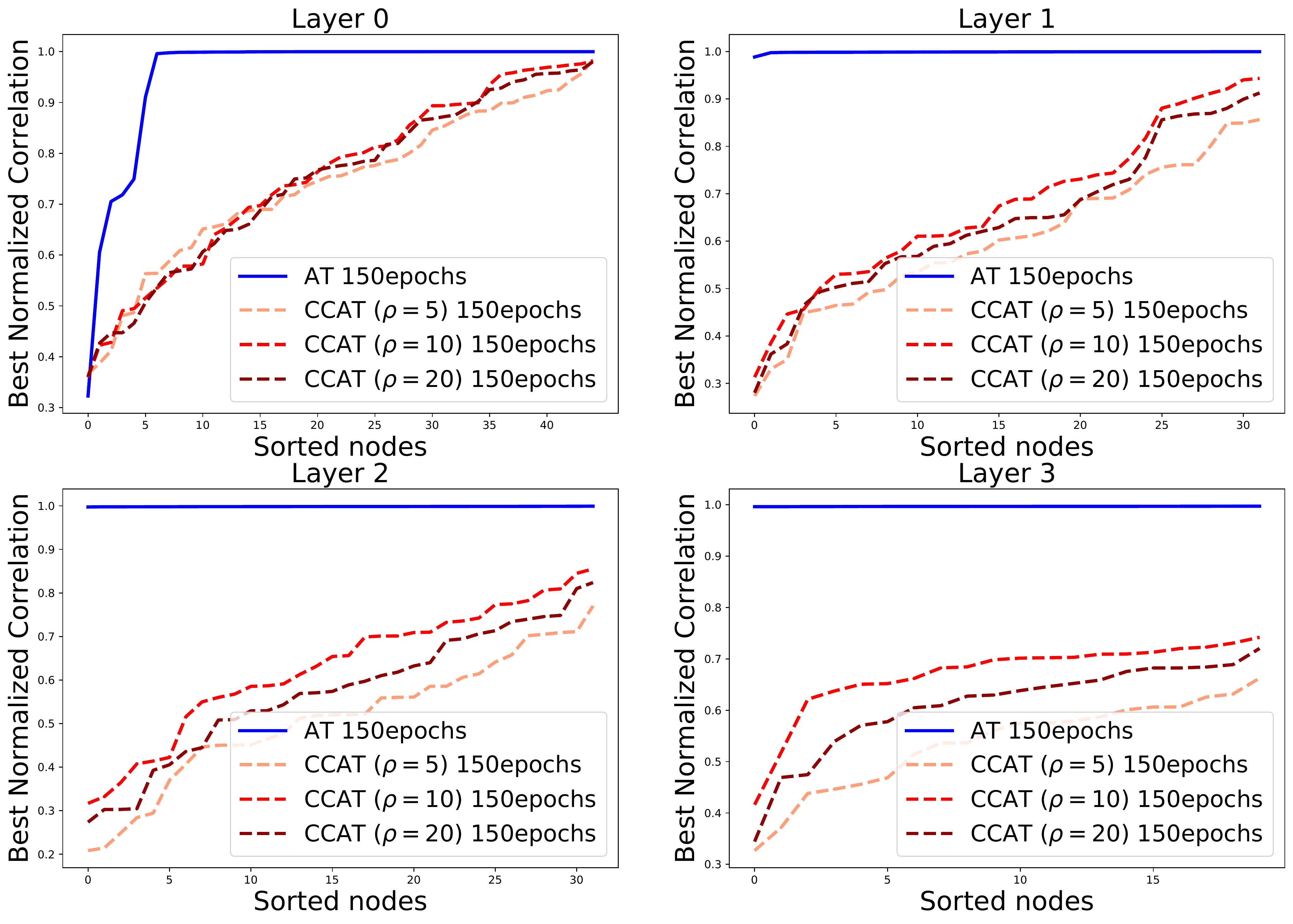}
%\vspace{-2em}
\caption{\small Sorted BNC curve of student trained by CCAT. Specialization correlates with robust accuracy (Table~\ref{tab:ccattable}).}
%AT refers to the vanilla \textbf{Adversarial Training}.}
%\vspace{-2em}
\label{fig:ccat}
\end{figure}

We also extend our analysis to other training techniques that improve the robustness of the model. \textbf{Confidence-Calibrated Adversarial Training (CCAT)}~\citep{stutz2019confidence} proposes to generate high confidence adversarial examples with calibrated soft labels. 
% and reject the ones with low confidence.  
Specifically, for an input $(\vx,\vy)$, adversarial example $\mathbf{x+\delta}$ is generated as:
%To further improve the robustness for adversarial training, 
%\begin{equation}
    $\delta =  \arg\max_{||\delta||_\infty \leq \epsilon} \max_{\mathbf{k} \neq \vy} f^C_{\mathbf{k}}(\vx+\delta)$,
%\end{equation}
where $f^C_{\mathbf{k}}(\vx)$ denotes model $f$'s output confidence on label $\mathbf{k}$, and $\epsilon$ denotes the tolerance of $\ell_\infty$ perturbation scale. The confidence parameter $\lambda(\delta)$ is decided by the $\ell_\infty$ norm of $\delta$ and the hyper-parameter $\rho$:
%\begin{small}{
%\begin{equation}
   $\lambda(\delta) = (1 - \min(1, ||\delta||_\infty/\epsilon))^\rho$,
  % $\lambda(\delta) = \bigg(1 - \min\bigg(1, \frac{||\delta||_\infty}{\epsilon}\bigg)\bigg)^\rho$
%\end{equation}
%}
%\end{small}
and the confidence-calibrated soft label $\tilde{\vy}$ is obtained by mixing the one-hot vector of the label $y$ with the confidence:
%\begin{equation}
    $\tilde{\vy} = \lambda(\delta)\text{one\_hot}(\vy) + (1 - \lambda(\delta))\frac{1}{K}$, 
%\end{equation}
where $K$ refers to the number of labels.

\begin{figure}[!t]
\centering
%\vspace{-0.2in}
\includegraphics[width=\linewidth]{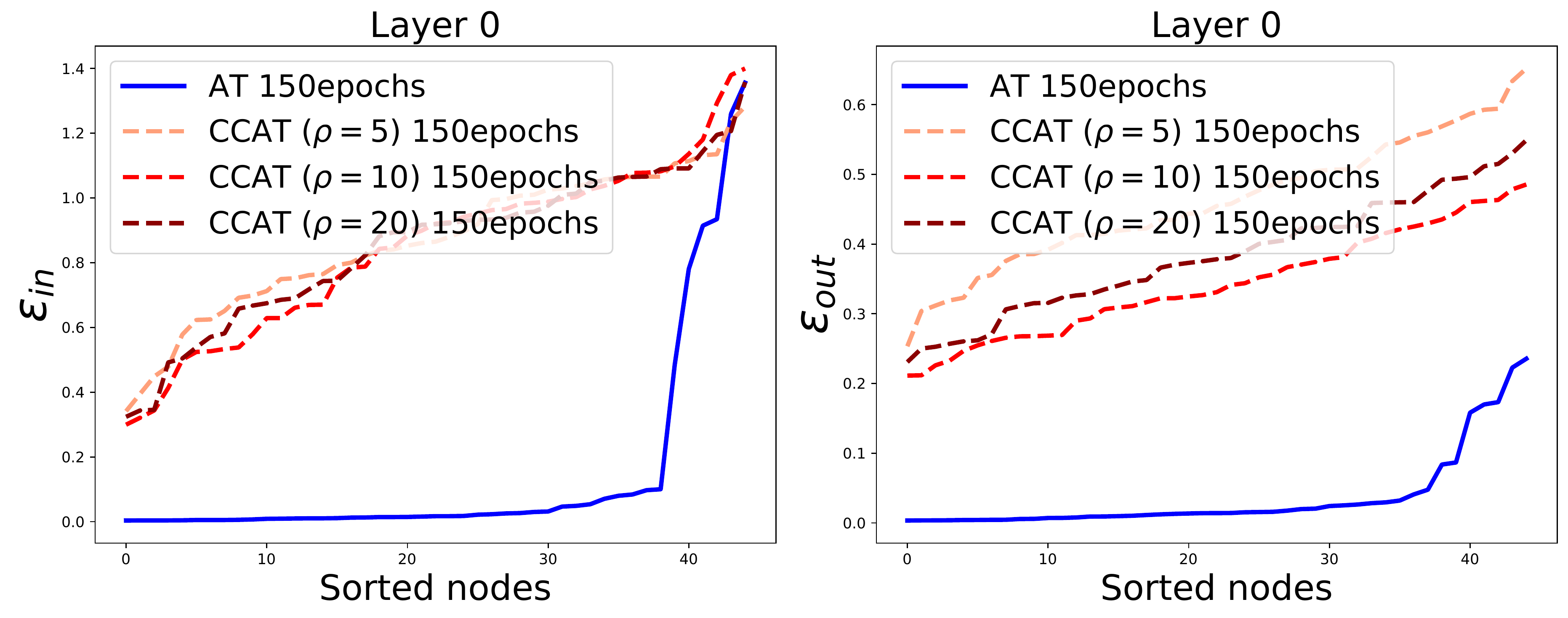}
%\vspace{-0.25in}
\caption{\small $(\epsilon_\tin,\epsilon_\tout)$ curve of student models trained by CCAT with different $\rho$. Here ``AT'' means that we use vanilla Adversarial Training with oracle-adversarial samples (Eqn.~\ref{eq:backprop-adv-oracle}), which leads to much better specialization.}
\label{fig:ccateps}
\end{figure}

In the teacher-student setting, we consider the \textbf{confidence} to be the $\ell_\infty$ distance between student and teacher's logit. We generate the high-confidence adversarial example by:
%\begin{equation}
$\delta = \arg\max_{||\delta||_\infty \leq \epsilon} \max_{\mathbf{k}} |s_{\mathbf{k}}(\vx+\delta)-t_{\mathbf{k}}(\vx+\delta))|$,
%\end{equation}
where $s_{\mathbf{k}}(\vx), t_{\mathbf{k}}(\vx)$ refer to the output logit of student and teacher on label $\mathbf{k}$ respectively. We apply the confidence-calibrated soft label $\tilde{\vy}$ to the adversarial examples and evaluate the robustness and specialization of CCAT models trained for $150$ epochs, with $\rho=5,10,20$.
\begin{table}[!htbp]
\centering
%\vspace{-1em}
\caption{\small Robust Accuracy (\%) of student models trained for $150$ epochs, with CCAT given different $\rho$ and AT.}
%\vspace{-0.5em}
\scalebox{0.7}{
\begin{tabular}{c|c|c|c|c}
\toprule
         Model       & CCAT ($\rho=5$) & CCAT ($\rho=10$) & CCAT ($\rho=20$) & AT      \\ \hline
Robust Acc & 47.25      & 52.04       & 49.33       & 84.07 \\ \bottomrule
\end{tabular}}
\vspace{-2em}
\label{tab:ccattable}
\end{table}
Figure~\ref{fig:ccat} and Table~\ref{tab:ccattable} show the neuron specialization and robustness of CCAT models respectively. 
We find that CCAT with $\rho=10$ achieves the best robustness and neuron specialization among all CCAT models. Again, we notice that there is a strong correlation between specialization and robustness (e.g., $\rho=10$ achieves the highest degree of specialization \emph{and} robustness, $\rho=20$ achieves the second highest on both, and similarly for $\rho=5$). Figure~\ref{fig:ccateps} shows the $\epsilon_{\text{in}}, \epsilon_{\text{out}}$ curves, and indicates AT can improve specialization for both in-plane and out-plane directions.

%We also notice that AT models could be much better in both robustness and specialization perspective.
%In the meantime, comparing with AT, the robustness of AT is much higher than the CCAT models, so is their neuron specialization.

\noindent\textit{Remarks.} Comparing different CCAT models with AT models, the results consistently show that the neuron specialization of student models is highly correlated with the robustness, which is aligned with our observation in Sec~\ref{sec-standard}.
In addition, AT models with better robustness may be due to the information loss during the confidence calibration: while the confidence calibration captures the balanced adversarial distribution, it will provide inconsistent confidence to the teacher's output. To align the confidence distribution with the teacher's output would be an interesting future work.
%In the meantime, the fact that with proper $\rho=10$ the CCAT model achieves the highest robustness is aligned with the original observation from the paper, and the \tsframework provides the justification for it from the neuron specialization perspective ($\rho=10$ achieves the highest neuron specialization among all CCAT models).

% correlations of three different CCAT-based student models to the teacher model. We can find that all of those alignments are quite low comparing to the standard adversarially trained model. We considered this as the information loss result from the confidence calibration procedure: while the confidence calibration procedure can indeed capture the balanced adversarial distribution, it will provide the inconsistent confidence to the teacher's output confidence. 

\vspace{-0.05in}
\section{Conclusion and Future Work}
\vspace{-0.05in}
In this paper, we leverage the \tsframework framework to study the model robustness and explain the origin of adversarial samples in a trained network. In our setting, we assume the labels to be the output of an \textit{oracle} teacher and student learns from the teacher through the teacher's output. In this setting, model vulnerability (and adversarial samples) naturally arise when the nodes (neurons) in a learned student do not fully reconstruct (or ``specialized into'') teacher's nodes when the input data are low-dimensional. Specifically, we theoretically show that, when training converges, student nodes are specialized in the low-dimensional input subspace, but may not be specialized out of such a subspace, leaving space for adversarial examples. Extensive experiments show a clear correlation between model robustness and degree of student specialization measured by normalized correlation between activations of teacher and student, in standard training, adversarial training (AT) and Confidence-Calibrated Adversarial Training (CCAT). Based on this new perspective, future work includes regularization of unspecialized student nodes during training, label-extrapolation of adversarial samples in AT, etc. 

\vspace{-0.5em}
\section{Acknowledgement}
\vspace{-0.5em}
This work is partially supported by NSF grant No.1910100 and DARPA QED-RML-FP-003 awarded to Bo Li.

% shows the alignment of the robust feature model to the teacher model from the correlation perspective. We can show that $150$ epochs finetuned robust feature model has way better correlations to the teacher than $150$ epochs regular trained model, and slightly better than $300$ epochs regular trained model. That shows the robust feature dataset can help the model to catch the in-place robust characteristic and improve the correlation between the student model and the teacher model.

% Table~\ref{tab:robustfeaturera} shows the comparison between the robust feature model and the regular trained model under the robustness perspective. we use the Robust accuracy as our metric. From the results, we can see even the robust feature model has better alignment to the teacher model comapring to the $300$ epochs regular trained model, the robustness is slightly worse. We consider this phenomenon resulting from the \textbf{low rank property} of the robust feature dataset. Once the model is trained with the low-rank dataset, the $\epsilon_{out}$ will increase and the model will show the vulnerability to the outer-place adversarial examples. How to generate the high-rank robust feature data could be a significant work in the future.~\xinyun{We want to compare with normal adversarial training, so we may want to discuss that adversarial examples are high-dimensional instances.}

% \bibliographystyle{abbrvnat}
%\newpage
%\bibliographystyle{apalike}
%\bibliography{main}

\newcommand\sbullet[1][.5]{\mathbin{\vcenter{\hbox{\scalebox{#1}{$\bullet$}}}}}

\clearpage

\newpage
\appendix
\onecolumn
\aistatstitle{Supplementary Materials of \\ Understanding Robustness in Teacher-Student Setting: A New Perspective}
\renewcommand\thesection{\Alph{section}}
\renewcommand\thesubsection{\thesection.\arabic{subsection}}
\section{Proofs}
\subsection{Lemma}
\begin{lemma}
\label{lemma:projected-vector}
If $\tilde\vx = U\tilde \vy + \tilde \vx_0$, where $U\in \rr^{d\times d'}$, then the inner product $\vw^\t\vx$ in the original space can be written as the inner product in the reduced space $\vw_y^\t\vy$ with 
\begin{equation}
    \vw^\t\vx = \vw_y^\t\vy, \quad\mathrm{for}\ \vw_y := 
    \left[
  \begin{array}{c}
    \tilde\vw_y \\
    b_y
  \end{array}
  \right] = 
    \left[
  \begin{array}{c}
    U^\t \tilde\vw \\
    \tilde\vw^\t\vx_0 + b
  \end{array}
  \right]
\end{equation}  
\end{lemma}
\begin{proof}
Since the augmented vector $\vx := [\tilde \vx; 1] \in \rr^{d+1}$, the inner product $\vw^\t\vx$ can be written as:
\begin{equation}
    \vw^\t\vx := \tilde\vw^\t\tilde\vx + b = \underbrace{\tilde\vw^\t U}_{\tilde\vw^\t_y}\tilde\vy + \underbrace{\tilde\vw^\t\vx_0 + b}_{b_y}
\end{equation}
and the conclusion follows.
\end{proof}

% \subsection{Theorem~\ref{thm:proj-specialization}}
\subsection{Theorem 1}
\begin{proof}
For low-dimensional input space $\cX$, we could always find a set of orthonormal bases $U = [\vu_1, \vu_2,\ldots, \vu_{d'}]$ so that for any point $\tilde\vx \in X$, we have $\tilde\vx = U\tilde\vy + \tilde\vx_0$. Therefore, by Lemma~\ref{lemma:projected-vector}, the inner product $\vw^\t\vx$ can be written as 
\begin{equation}
    \vw^\t\vx = \vw_y^\t\vy, \quad\mathrm{for}\ \vw_y := 
    \left[
  \begin{array}{c}
    \tilde\vw_y \\
    b_y
  \end{array}
  \right] = 
    \left[
  \begin{array}{c}
    U^\t \tilde\vw \\
    \tilde\vw^\t\vx_0 + b
  \end{array}
  \right]
\end{equation}  

Then $\vy$ is full-rank in $X$ and we can apply Lemma 3 in~\citet{tian2019student} for the reduced space of $\vy$ to draw the conclusion that for each teacher node $j$ whose boundary is observed by a student node $k$ with $\alpha_{jk} \neq 0$, there exists at least one student node $k'$ so that $\vw^*_{y,j} = \lambda\vw_{y,k}$ with $\lambda > 0$. Taking its first $d'$ components, we have $U^\t \tilde \vw^*_j = \lambda U^\t \tilde \vw_k$. Notice that $\proj_{\cX}[\tilde \vw_j^*] = UU^\t\tilde\vw_j^*$, we have $\proj_{\cX}[\tilde \vw_j^*] = \lambda\proj_{\cX}[\tilde \vw_k]$. 
\end{proof}

\subsection{Lemma~\ref{lemma:relation}} 
\begin{lemma}[Relation between Hyperplanes (Lemma 5 in ~\citet{tian2019student})]
\label{lemma:relation}
Let $\vw_j$ and $\vw_{j'}$ be two distinct hyperplanes with $\|\tilde\vw_j\| = \|\tilde\vw_{j'}\| = 1$. Denote $\theta_{jj'}$ as the angle between the two vectors $\vw_j$ and $\vw_{j'}$. Then there exists $\tilde\vu_{j'} \perp \tilde\vw_j$ and $\vw_{j'}^\t \tilde\vu_{j'} = \sin\theta_{jj'}$.
\end{lemma}
\subsection{Lemma~\ref{lemma:evidence-misalignment}}
\begin{lemma}[Evidence of Data points on Misalignment]
\label{lemma:evidence-misalignment}
Let $R \subset \rr^d$ be an open set. Consider $K$ ReLU nodes $f_j(\vx) = \sigma(\vw_j^\t\vx)$, $j = 1, \ldots, K$. $\|\tilde\vw_j\| = 1$, $\vw_j$ are not co-linear. Then for a node $j$ with $\partial E_j \cap R \neq \emptyset$, either of the conditions holds:
\begin{itemize}
  \setlength\itemsep{0.2em}
    \item[\textbf{(1)}] There exists node $j'\neq j$ so that $\sin\theta_{jj'} \le MK\epsilon/|c_j|$ and $|b_{j'} - b_j| \le M_2\epsilon/|c_j|$.
    \item[\textbf{(2)}] There exists $\vx_j \in \partial E_j \cap R$ so that for any $j'\neq j$, $|\vw^\t_{j'}\vx_j| > 5\epsilon/|c_j|$.
\end{itemize}
\vspace{-0.6em}

\newpage
where: 
\vspace{-0.6em}
\begin{itemize}
  \setlength\itemsep{0.05em}
    \item $\theta_{jj'}$ is the angle between $\tilde\vw_j$ and $\tilde\vw_{j'}$,
    \item $r$ is the radius of a $d-1$ dimensional ball contained in $\partial E_j \cap R$, 
    \item $M = \frac{10}{r}\sqrt{\frac{d}{2\pi}}$, $M_0 = \max_{\vx\in \partial E_j \cap R} \|\vx\|$ and $M_2 = 2M_0MK + 5$.
\end{itemize}
\end{lemma}
\vspace{-0.6em}
\begin{proof}
Define $q_j = 5\epsilon / |c_j|$. For each $j' \neq j$, define $I_{j'} = \{\vx: |\vw_{j'}^\t\vx|\le q_j,\ \vx \in \partial E_j\}$. We prove by contradiction. Suppose for any $j'\neq j$,  $\sin\theta_{jj'} > KM\epsilon/|c_j|$ or $|b_{j'} - b_j| > M_2 \epsilon/|c_j|$. Otherwise the theorem already holds.

\textbf{Case 1. When $\sin\theta_{jj'} > KM\epsilon/|c_j|$ holds}. 

From Lemma~\ref{lemma:relation}, we know that for any $\vx \in \partial E_j$, if $\vw_{j'}^\t\vx = -q_j$, with $a_{j'} \le \frac{2q_j|c_j|}{MK\epsilon} = \frac{10}{MK}$, we have $\vx' = \vx + a_{j'}\vu_{j'}\in \partial E_j$ and $\vw_{j'}^\t\vx' = +q_j$. 

Consider a $d - 1$-dimensional sphere $B \subseteq \Omega_j$ and its intersection of $I_{j'} \cap B$ for $j'\neq j$. Suppose the sphere has radius $r$. For each $I_{j'} \cap B$, its $d-1$-dimensional volume is upper bounded by:
\begin{equation}
    V(I_{j'}\cap B) \le a_{j'} V_{d-2}(r) \le \frac{10}{MK}V_{d-2}(r)
\end{equation}
where $V_{d-2}(r)$ is the $d-2$-dimensional volume of a sphere of radius $r$. Intuitively, the intersection between $\vw_{j'}^\t \vx = -q_j$ and $B$ is at most a $d-2$-dimensional sphere of radius $r$, and the ``height'' is at most $a_{j'}$. 

\textbf{Case 2. When $\sin\theta_{jj'} \le KM\epsilon/|c_j|$ but $|b_{j'} - b_j| > M_2 \epsilon/|c_j|$ holds.}

In this case, we want to show that for any $\vx\in \Omega_j$, $|\vw_{j'}^\t \vx| > q_j$ and thus $I_{j'} \cap B = \emptyset$. If this is not the case, then there exists $\vx \in \Omega_j$ so that $|\vw_{j'}^\t \vx| \le q_j$. Then since $\vx \in \partial E_j$, we have:
\begin{equation}
\begin{aligned}
    |\vw_{j'}^\t \vx| &= |(\vw_{j'}-\vw_j)^\t\vx| = |(\tilde\vw_{j'}-\tilde\vw_j)^\t\tilde\vx + (b_j' - b_j)| \le q_j
\end{aligned}
\end{equation}
Therefore, from Cauchy inequality and triangle inequality, we have:
\begin{equation}
\begin{aligned}
\|\tilde\vw_{j'}-\tilde\vw_j\|\|\tilde\vx\| &\ge
|(\tilde\vw_{j'}-\tilde\vw_j)^\t\tilde\vx| \ge |b_j' - b_j| - |\vw_{j'}^\t \vx|
\end{aligned}
\end{equation}
From the condition, we have $\|\tilde\vw_{j'}-\tilde\vw_j\| = 2\sin\frac{\theta_{jj'}}{2} \le 2\sin\theta_{jj'} \le 2KM\epsilon/|c_j| $. Then
\begin{equation}
\begin{aligned}
2M_0MK\epsilon/|c_j| &\ge |(\tilde\vw_{j'} - \tilde\vw_j)^\t \tilde\vx| \ge |b_{j'} - b_j| - q_j > M_2 \epsilon/|c_j| - 5\epsilon/|c_j|
\end{aligned}
\end{equation}
which is equivalent to:
\begin{equation}
2M_0MK > M_2 - 5
\end{equation}
which means that 
\begin{equation}
M_2 < 2M_0MK + 5
\end{equation}
This is a contradiction. Therefore, $I_{j'} \cap B = \emptyset$ and thus $V(I_{j'} \cap B) = 0$.

\textbf{Volume argument.} Therefore, from the definition of $M$, we have $V(B) = V_{d-1}(r) \ge r\sqrt{\frac{2\pi}{d}} V_{d-2}(r) = \frac{10}{M}V_{d-2}(r)$, then we have:
\begin{equation}
%\begin{aligned}
    V(B) \ge \frac{10}{M}V_{d-2}(r) > (K-1) \cdot \frac{10}{MK} V_{d-2}(r)\ge \sum_{j'\neq j, j'\ \mathrm{in\ case\ 1}} V(I_{j'} \cap B)  
%\end{aligned}
\end{equation}
This means that there exists $\vx_j \in B \subseteq \Omega_j$ so that $\vx_j \notin I_{j'} \cap B$ for any $j' \neq j$ and $j'$ in case $1$. That is, 
\begin{equation}
    |\vw_{j'}^\t\vx_j| > q_j
\end{equation}
On the other hand, for $j'$ in case $2$, the above condition holds for entire $\Omega_j$, and thus hold for the chosen $\vx_j$.
\end{proof}

\begin{figure*}
    \centering
    \includegraphics[width=\textwidth]{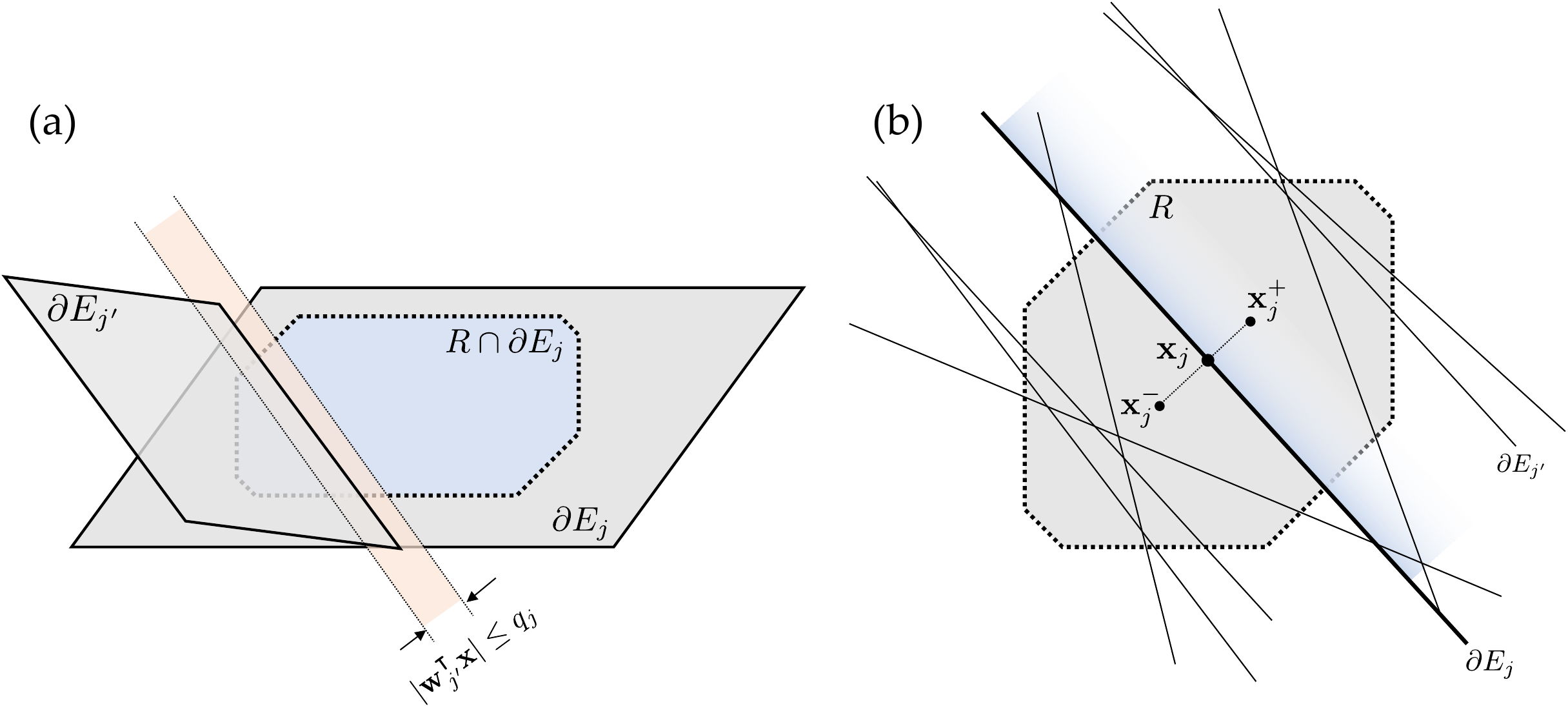}
    \caption{\textbf{(a)} Lemma~\ref{lemma:evidence-misalignment}. \textbf{(b)} Lemma~\ref{lemma:local-colinear-noisy-weight}.}
    \label{fig:lemma6}
\end{figure*}

\subsection{Lemma~\ref{lemma:local-colinear-noisy-weight}} 
\begin{lemma}[Local ReLU Independence, Noisy case]
\label{lemma:local-colinear-noisy-weight}
Let $R$ be an open set. Consider $K$ ReLU nodes $f_j(\vx) = \sigma(\vw_j^\t\vx)$, $j = 1, \ldots, K$. $\|\tilde\vw_j\| = 1$, $\vw_j$ are not co-linear. If there exists $c_1, \ldots, c_K, c_{\sbullet}$ and $\epsilon$ so that the following is true:
\begin{equation}
    \bigg|\sum_j c_j f_j(\vx) + c_{\sbullet} \vw^\t_{\sbullet}\vx \bigg| \le \epsilon, \quad \forall \vx \in R
\end{equation}
and for a node $j$, $\partial E_j \cap R \neq \emptyset$. Then there exists node $j'\neq j$ so that $\sin\theta_{jj'} \le MK\epsilon/|c_j|$ and 
$|b_{j'} - b_j| \le M_2\epsilon/|c_j|$, where $r, M, M_2$ are defined in Lemma~\ref{lemma:evidence-misalignment} but with $r' = r - 5\epsilon/|c_j|$.
\end{lemma}
\begin{proof}
Let $q_j = 5\epsilon/|c_j|$ and $\Omega_j = \{\vx : \vx \in \partial E_j \cap R,\ \   B(\vx, q_j) \subseteq R\}$. If situation (1) in Lemma~\ref{lemma:evidence-misalignment} happens then the theorem holds. Otherwise, applying Lemma~\ref{lemma:evidence-misalignment} with $R' = \{\vx: \vx \in R, \ \ B(\vx, q_j)\subseteq R\}$ and there exists $\vx_j \in \Omega_j$ so that  
\begin{equation}
    |\vw_{j'}^\t\vx_j| \ge q_j = 5\epsilon/|c_j| 
\end{equation}
Let two points $\vx^{\pm}_{j} = \vx_j \pm q_j \tilde\vw_j \in R$. In the following we show that the three points $\vx_j$ and $\vx^\pm_j$ are on the same side of $\partial E_{j'}$ for any $j'\neq j$. This can be achieved by checking whether $(\vw_{j'}^\t \vx_{j})(\vw_{j'}^\t \vx^\pm_{j}) \ge 0$ (Figure~\ref{fig:lemma6}):
\begin{small}{
\begin{eqnarray}
(\vw_{j'}^\t \vx_{j})(\vw_{j'}^\t \vx^\pm_{j}) &=& (\vw_{j'}^\t \vx_{j})\left[\vw_{j'}^\t (\vx_{j} \pm q_j \tilde\vw_j)\right] \\
&=& (\vw_{j'}^\t \vx_{j})^2 \pm q_j (\vw_{j'}^\t \vx_{j}) \vw_{j'}^\t\tilde\vw_j \\
&=& |\vw_{j'}^\t\vx_j| (|\vw_{j'}^\t\vx_j| \pm q_j  \vw_{j'}^\t\tilde\vw_j)
\end{eqnarray}}
\end{small}

Since $|\vw_{j'}^\t\tilde\vw_j| \le 1$, it is clear that $(\vw_{j'}^\t \vx_{j})(\vw_{j'}^\t \vx^\pm_{j}) \ge 0$. Therefore the three points $\vx_j$ and $\vx^\pm_j$ are on the same side of $\partial E_{j'}$ for any $j' \neq j$. 

Let $h(\vx) = \sum_j c_j f_j(\vx) + c_{\sbullet} \vw^\t_{\sbullet} \vx$, then $|h(\vx)| \le \epsilon$ for $\vx \in R$. Since $\vx^+_j + \vx^-_j = 2\vx_j$, we know that all terms related to $\vw_{\sbullet}$ and $\vw_{j'}$ with $j\neq j$ will cancel out (they are in the same side of the boundary $\partial E_{j'}$) and thus:
\begin{equation}
\begin{aligned}
    4\epsilon &\ge |h(\vx^+_j) + h(\vx^-_j) - 2h(\vx_j)| = |c_jq_j\vw^\t_j\vw_j| = |c_j|q_j = 5\epsilon 
\end{aligned}
\end{equation}
which is a contradiction. 
\end{proof}

% \subsection{Theorem~\ref{thm:proj-specialization-noise}}
\subsection{Theorem 2}
\label{sec-proof2}
\begin{proof}
Note that from Theorem 1, any input $\vx \in \cU\cap R$ can be written as $\tilde\vx = U\tilde\vy + \tilde\vx_0$, where $U\in \rr^{d\times d'}$ is a column-orthogonal matrix (i.e, $U^\t U = I_{d'\times d'}$). Also from Lemma~\ref{lemma:projected-vector}, any inner-product $\vw^\t\vx$ can be written as $\vw_y^\t\vy$, with $\vw_y := [\tilde\vw_y; \tilde\vw^\t\vx_0 + b]$ and $\tilde\vw_y := U^\t \tilde\vw$, and the inner product of two projected weights is:
% Note that from Theorem~\ref{thm:proj-specialization}, any input $\vx \in \cU\cap R$ can be written as $\tilde\vx = U\tilde\vy + \tilde\vx_0$, where $U\in \rr^{d\times d'}$ is a column-orthogonal matrix (i.e, $U^\t U = I_{d'\times d'}$). Also from Lemma~\ref{lemma:projected-vector}, any inner-product $\vw^\t\vx$ can be written as $\vw_y^\t\vy$, with $\vw_y := [\tilde\vw_y; \tilde\vw^\t\vx_0 + b]$ and $\tilde\vw_y := U^\t \tilde\vw$, and the inner product of two projected weights is:
\begin{equation}
\begin{aligned}
    \tilde\vp_j^\t \tilde\vp_k &:= \proj_{\cU}[\tilde\vw_j]^\t \proj_{\cU}[\tilde\vw_k] = \tilde\vw_j UU^\t UU^\t \tilde\vw_k = \tilde\vw_j U U^\t \tilde\vw_k = \tilde\vw_{y,j}^\t\tilde\vw_{y,k}
\end{aligned}
\end{equation}
Therefore, all the ReLU activations can be written in the reduced space, and the projected angle $\theta^{\cU}_{jk} := \arccos \tilde\vp_j^\t \tilde\vp_k$ we are aiming for is also defined in the reduced space $\vy$. Applying Lemma~\ref{lemma:local-colinear-noisy-weight} on the reduced space $\vy$ with $r = r(\cU\cap R \cap \partial E_j)$, and the conclusion follows.
\end{proof}

% \subsection{Corollary~\ref{thm:proj-unspecialized}}
\subsection{Corollary 1}
\begin{proof}
By Theorem 2, we know that for a node $k_0$, if it is observed by another student node $k$, then there exists a node $j$ (can be either a teacher or another student node) so that their projected angle $\sin\theta^\cU_{jk_0}$ has the following upper bound: 
% By Theorem~\ref{thm:proj-specialization-noise}, we know that for a node $k_0$, if it is observed by another student node $k$, then there exists a node $j$ (can be either a teacher or another student node) so that their projected angle $\sin\theta^\cU_{jk_0}$ has the following upper bound: 
\begin{equation}
    \sin\theta^\cU_{jk_0} \le MK\epsilon / |\alpha_{kk_0}| 
\end{equation}
where $\alpha_{kk_0} := \vv_k^\t\vv_{k_0}$, $\vv_k \in \rr^C$ is the fan-out weights, and $C$ is the number of output for the two-layer network. On the other hand, by the condition, we have $\sin\theta^\cU_{jk_0} \ge c_0$ for any other teacher and student nodes, including $j$. Therefore, we have:
\begin{equation}
    c_0 \le \sin\theta^\cU_{jk_0} \le MK\epsilon / |\alpha_{kk_0}| 
\end{equation}
which leads to
\begin{equation}
    |\vv_k^\t\vv_{k_0}| = |\alpha_{kk_0}| \le MK\epsilon / c_0 
\end{equation}
If the student node $k_0$ is observed by $C$ independent observers $k_1,k_2,\ldots,k_C$, then we have:
\begin{equation}
    |\vv_{k_m}^\t\vv_{k_0}| = |\alpha_{k_0 k_m}| \le MK\epsilon / c_0,\quad m=1,\ldots,C 
\end{equation}
Let $Q := [\vv_{k_1},\vv_{k_2},\ldots,\vv_{k_C}] \in \rr^{C\times C}$, then we have $\|Q^\t\vv_{k_0}\|_\infty \le MK\epsilon / c_0$ and:
\begin{equation}
    \|\vv_{k_0}\|_\infty \le \|Q^{-\t}\|_\infty \|Q^\t\vv_{k_0}\|_\infty \le \|Q^{-1}\|_1 MK\epsilon / c_0 
\end{equation}
where $\|\cdot\|_1$ is the 1-norm of a matrix (or maximum absolute row sum). 
\end{proof}

% \subsection{Theorem~\ref{thm:multiple-teachers}}
\subsection{Theorem 3}
\begin{proof}
Note that according to Lemma 1 in~\citet{tian2019student} (Appendix B.1), for any teacher $f^{*m}$ and any student $f$ of the same depth, we have at layer $l = 1$:

\begin{small}{
\begin{eqnarray}
\vg_1(\vx) &=& D_1(\vx)V_1^\t(\vx)\left[V^{*m}_1(\vx)\vf^{*m}_1(\vx) - V_1(\vx)\vf_1(\vx)\right]  \\
&=& D_1(\vx)V_1^\t(\vx) (\vy^{*m}(\vx) - \vy(\vx))
\end{eqnarray}}
\end{small}

since for two-layer network, we have $\vy(\vx) = V_1(\vx)\vf_1(\vx)$ is the output. Therefore, if the gradient computed between teacher $f^*$ and student $f$ has $\|\vg_1\|_\infty \le \epsilon$, then

\begin{small}{
\begin{eqnarray}
   \|\vg^m_1\|_\infty &=& \|D_1V_1^\t(\vy^{*m}(\vx) - \vy(\vx))\|_\infty \\ 
   &\le&\|D_1V_1^\t(\vy^{*m}(\vx) - \vy^*(\vx))\|_\infty + \|D_1V_1^\t(\vy^*(\vx) - \vy(\vx))\|_\infty \\  
   &\le&\|D_1V_1^\t(\vy^{*m}(\vx) - \vy^*(\vx))\|_\infty + \|\vg_1\|_\infty\\  
   &\le& \|V_1\|_1 \epsilon_0 + \epsilon 
\end{eqnarray}}
\end{small}

where $\|V_1\|_1 = \max_j \|\vv_j\|_1$ is the 1-norm (or the maximum absolute row sum) of matrix $V_1$. Then we apply Theorem 2 between the student $f$ and teacher $f^{*m}$ and the conclusion follows.
% where $\|V_1\|_1 = \max_j \|\vv_j\|_1$ is the 1-norm (or the maximum absolute row sum) of matrix $V_1$. Then we apply Theorem~\ref{thm:proj-specialization-noise} between the student $f$ and teacher $f^{*m}$ and the conclusion follows.
\end{proof}

\subsection{Unidentifiable teachers and Student Bias} 
\vspace{-0.1in}
We might wonder what would happen if there exist two teachers $f^{*1} \neq f^{*2}$ so that $\vy_i = f^{*1}(\vx_i) + \xi^1_i = f^{*2}(\vx_i) + \xi^2_i$ with different bias: $\|\xi^1_i\| \le \epsilon_0$ and $\|\xi^2_i\|\le \epsilon_0$. In this case, which teacher the student would converge into? We could use the same framework to analyze it: 
\begin{theorem}
\label{thm:multiple-teachers}
For any two-layered network $f^{*l}$ of the same architecture as $f^*$ and $\|f^*(\vx) - f^{*l}(\vx)\| \le \epsilon_0$ for all $\vx \in R$, when $\|\vg_1\|_\infty\le \epsilon$, for a teacher node $j$ in $f^{*l}$ observed by a student $k$, there exists a student $k'$ so that $\sin\theta^\cU_{jk'} \le MK(\epsilon + \epsilon_0 \max_j \|\vv_j\|_1) / \alpha^l_{jk}$.
\end{theorem}
Note that this theorem can be applied to any teacher $f^{*l}$ to yield a separate bound for the alignment. Some bounds are strong while others are loose. The larger $\alpha^l_{jk}$, the tighter the bound. Therefore, there are two phases in the training: \textbf{(1)} at the early stage of training, $\epsilon$ is fairly large, the norm of the fan-out weights $\|\vv_j\|_1$ is small, and many candidate teachers (as well as their hidden nodes) with reasonable $\epsilon_0$ can stand out as long as their $\alpha^l_{jk}$ is large. Therefore, the student moves to salient (large $\alpha^l_{jk}$) but potentially biased (large $\epsilon_0$) explanation. \textbf{(2)} When the training converges and $\epsilon$ is small, some $\|\vv_j\|_1$ becomes large, the ``real'' teacher with small bias $\epsilon_0$ gives the tightest bound, and the student converges to it.

The case (1) is interesting since it shows that the student node doesn't go straight to the ground truth teacher node from the beginning, but has a bias towards simple models that could roughly explain data (with reasonable $\epsilon_0$). This is a fixed bias for student nodes that only dependent on the dataset and regardless of the model initialization. This could be used to explain the adversarial transferability~\citep{goodfellow2014explaining}. In this paper, we focus on the specialization of student nodes on a specific teacher network and leave the case of ``one student multiple teachers'' (i.e., Theorem~\ref{thm:multiple-teachers}) for future empirical study.

\subsection{Ablation study on specialization distribution among teacher nodes}
To investigate how well one teacher node could be specialized by student nodes and the existence of special teacher nodes which are easy to be specialized by student nodes, we conduct the ablation study by training three student networks with different random initialization and check the number of student nodes specialized to each teacher node as shown in Figure~\ref{fig:histogram}. We found that teacher nodes are specialized almost uniformly by different student nodes, showing that there may not be special ``robust" teacher nodes, which could be an interesting finding.

\begin{figure*}
    \centering
    \includegraphics[width=\textwidth]{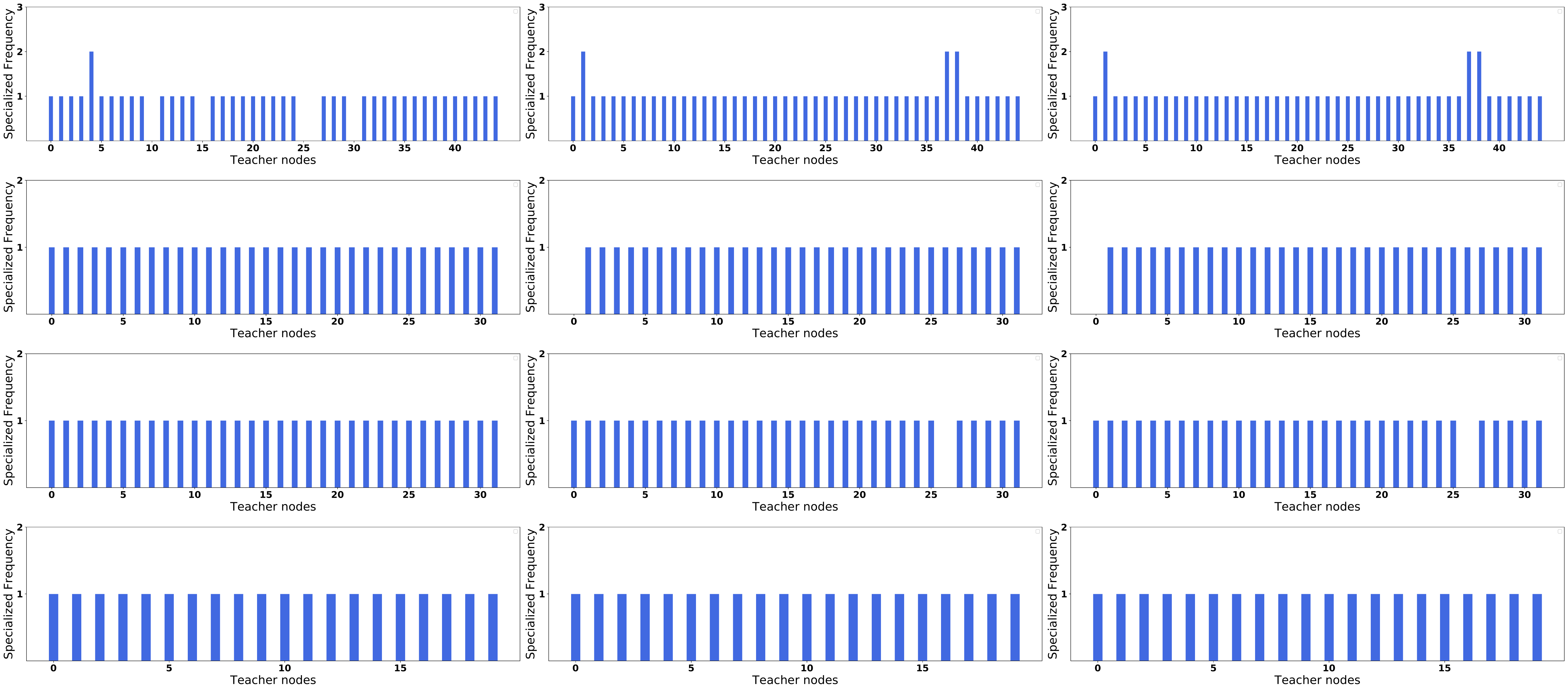}
    \caption{Specialized frequency of each teacher node by different student networks among different layers (We consider the node to be specialized if the NC is larger than 0.9). Figures in different columns refer to specialized evaluation with student network trained from different random initialization, while in different rows refer to evaluation on different layers. Here we can see 1) Teacher nodes are specialized uniformly by student nodes. 2) Different random initialization will lead to similar observations. }
    \label{fig:histogram}
\end{figure*}

\subsection{Analysis on Robust feature dataset}
\label{sec-robustfeature}

Robust feature disentanglement, proposed by \citet{ilyas2019adversarial}, is a general method to generate a robust feature dataset from a robustly trained model. Specifically, the robust feature dataset $\mathcal{D}=\{x_r\}$ is generated by minimizing the feature representation distance as below:
\begin{equation}
    x_r = \arg \min\nolimits_{x_r} ||f_M(x) - f_M(x_r)||_2
\end{equation} 
while $f_M$ represents the representation output of model $f$ and $x$ is drawn from the raw dataset. 
For every $x$ as the target image, the robust feature image $x_r$ is optimized from a randomly selected image or random noise.

In \tsframework setting, we define $f$ to be a robust student model if its prediction can be consistent with the teacher's prediction against oracle-adversarial or data-adversarial. Different from the standard setting, the generated $x_r$ may lie in different categories with $x$ from the teacher's perspective.
%The main difference of robust feature dataset generation between the standard and \tsframework setting is that the label of $x_r$ needs to be queried from the teacher rather than using the default ground-truth.
% can share the same label of the target image $x$ from human inception under the standard setting. However, the teacher model might consider $x$ and $x_r$ to have different label under the teacher-student setting.
In order to avoid the inconsistency, we add another term into robust feature generation's goal to minimize the logit difference between robust feature image $x_r$ and target image $x$ given by the teacher model:
% If we continue to use target image $x$'s label for robust feature image $x_r$, it will cause the inconsistency of the prediction between the teacher model and the robust feature trained student model and disturb the alignment. Thus we need to fix this issue and form a new optimization goal $J_{\text{ts}}$ defined by:
\begin{equation}
    x_r = \arg\min\nolimits_{x_r}  \alpha||L_t(x_r) - L_t(x)||_2 + ||f_M(x) - f_M(x_r)||_2
\end{equation}

where $\alpha$ is the balancing hyperparameter. We choose $\alpha = 0.5$ for the default setting.

We choose the AT model trained with $150$ epochs and generate the corresponding robust feature dataset $\mathcal{D}$. Based on $\mathcal{D}$, we train the robust feature model for $150$ epochs via fine-tuning on top of a $150$ epochs trained ST model.
% Since the robust feature model actually received the teacher model's logit feedback for $300$ epochs, 
In order to make a fair comparison, we compare the $150$ epochs Robust Feature Training (RFT) model to $150, 300$ epochs trained ST models and $150$ epochs AT model. All models are trained with the teacher's logit feedback.

\begin{table}[!htbp]
\centering
\caption{\small Robustness of student models trained with Robust Feature Training (RFT), Standard Training (ST), and Adversarial Training (AT) for different epochs.}
\begin{tabular}{l|l|l|l|l}
\toprule
          Model      & AT (150 epochs)        & ST ($150$ epochs) & ST ($300$ epochs) & RFT ($150$ epochs) \\ \hline
Robust Accuracy & \multicolumn{1}{c|}{83.27\%} & \multicolumn{1}{c|}{35.88\%}  & \multicolumn{1}{c|}{61.73\%}  & \multicolumn{1}{c}{45.39\%}         \\ \bottomrule
\end{tabular}

\label{tab:robustfeaturera}
\end{table}

\begin{figure}
\centering
\begin{minipage}{.47\textwidth}
\centering
    \includegraphics[width=1.0\linewidth]{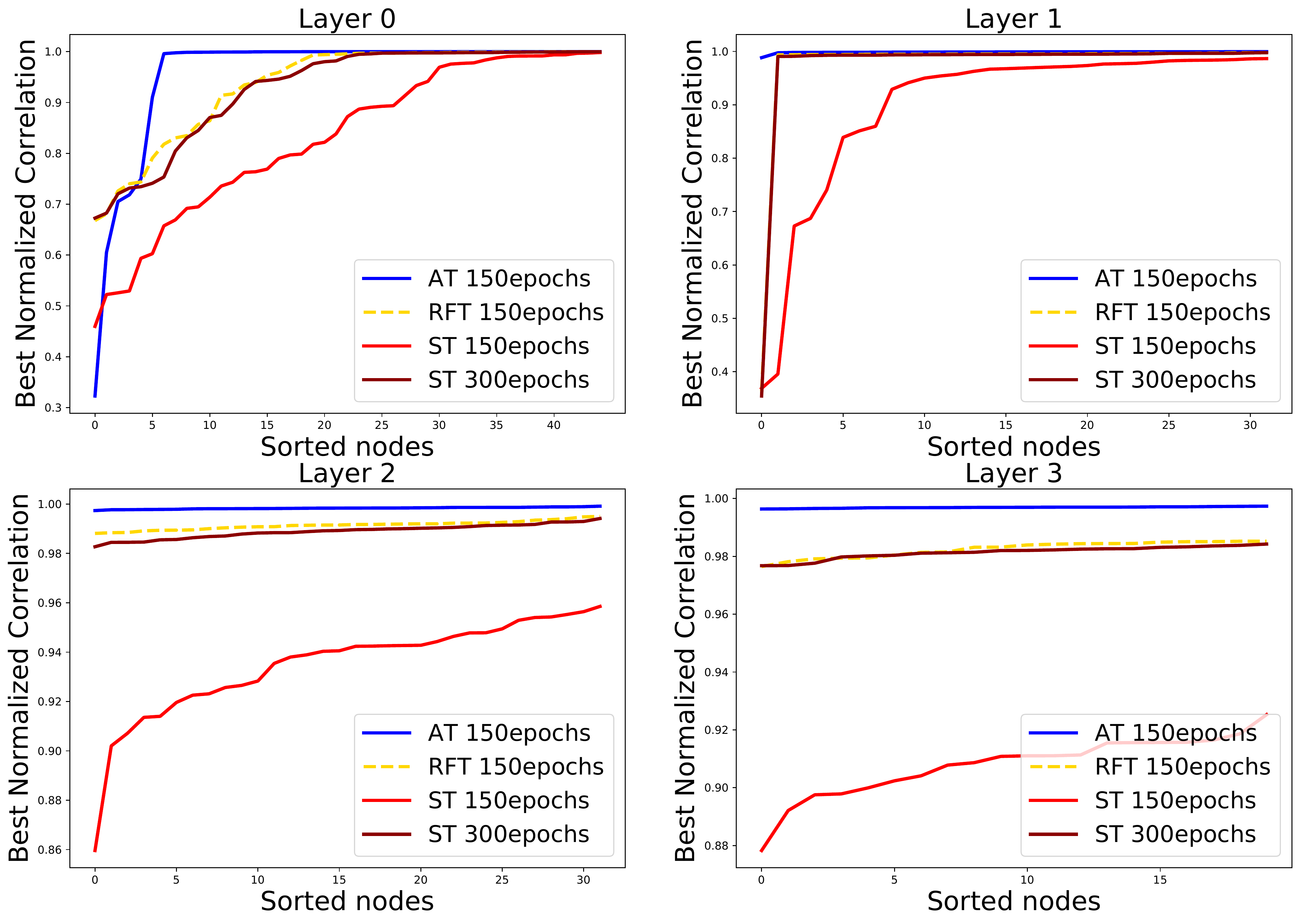}
    \caption{\small Sorted BNC curve of student models trained with Robust Feature Training (RFT), Standard Training (ST), and Adversarial Training (AT) trained for different epochs.}
    \label{fig:robustfeature}
\end{minipage}
\hspace{\fill}
\begin{minipage}{.47\textwidth}
\centering
    \includegraphics[width=1.0\linewidth]{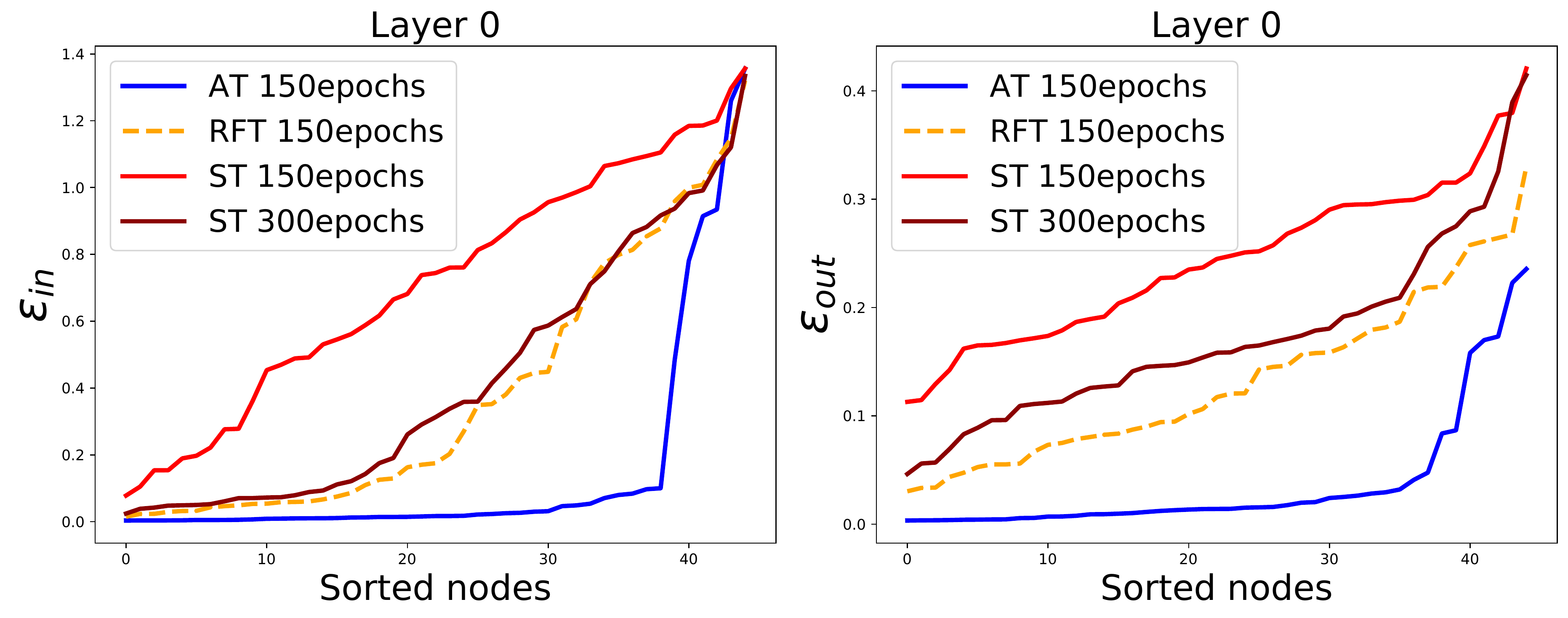}
    \caption{\small $(\epsilon_{\text{in}},\epsilon_{\text{out}})$ curve of student models trained with Robust Feature Training (RFT), Standard Training (ST), and Adversarial Training (AT) trained for different epochs.}
    \label{fig:robustfeatureeps}
\end{minipage}
\end{figure}

Table~\ref{tab:robustfeaturera} and Figure~\ref{fig:robustfeature}
show the robustness and neuron specialization of student models with RFT, ST, and AT. We can see \textbf{1)} AT model achieves the best robustness as well as the best neuron specialization; \textbf{2)} RFT model ($150$ epochs) fine-tuned from ST model ($150$ epochs) achieves better model robustness and specialization than ST model ($150$ epochs); \textbf{3)} The neuron specialization of RFT model ($150$ epochs) and ST model ($300$ epochs) is close but ST model ($150$ epochs) achieves better robustness. Figure~\ref{fig:robustfeatureeps} shows the $\epsilon_{\text{in}}, \epsilon_{\text{out}}$ curve and we can see the $150$ epochs RFT model shows the similar $\epsilon_{\text{in}}$ curve but slightly better $\epsilon_{\text{out}}$ curve to $300$ epochs ST model. We analyze this phenomenon by considering the robust feature dataset mainly captures the out-plane vulnerability. As we discussed in Section 5.3, the in-plane vulnerability could be more severe to the model's robustness and that could be the reason why the $150$ epochs RFT model achieves slightly worse robustness than the $300$ epochs ST model.

\noindent\textit{Remarks.}
Based on the comparison between RFT, AT, ST models, we can conclude \textit{again} that the neuron specialization of student models highly indicates their robustness. 
On the other hand, when the neuron specialization is close, the robustness comparison between them is less informative since other factors such as data distribution may have an impact on it.
In addition, the \tsframework provides an in-depth explanation of why the robust feature dataset exists from the neuron specialization perspective. The robust feature dataset can help model capture the in-plane data projection and out-plane vulnerability therefore improve the correlation between student and teacher, which leads to better model robustness. 
\end{document}